\documentclass{article} 
\usepackage{iclr2027_conference,times}

\usepackage{amsmath,amsfonts,bm}

\def\eqref#1{equation~\ref{#1}}

\def\1{\bm{1}}

\DeclareMathAlphabet{\mathsfit}{\encodingdefault}{\sfdefault}{m}{sl}
\SetMathAlphabet{\mathsfit}{bold}{\encodingdefault}{\sfdefault}{bx}{n}

\usepackage{amsthm}
\newtheorem{proposition}{Proposition}

\usepackage{hyperref}
\usepackage{url}
\usepackage{caption}
\usepackage{subcaption}

\usepackage{graphicx}
\usepackage{booktabs}
\usepackage{comment}
\usepackage{algorithm}
\usepackage{algorithmic}

\title{Efficient Pretraining of LLMs through \\Truncated SVD Representations}

\author{Kaivan Kamali \thanks{Corresponding author. Email: \texttt{kaivan.kamali@cognizant.com}} \\
  Cognizant AI Lab \\
  San Francisco, CA 94105, USA \\
  \texttt{kaivan.kamali@cognizant.com} \\
  \And
  \hspace{-1.3cm}Kajetan Schweighofer\thanks{Equal contribution}\\
  \hspace{-1.3cm}Cognizant AI Lab \\
  \hspace{-1.3cm}San Francisco, CA 94105, USA \\
  \hspace{-1.3cm}\texttt{kai.schweighofer@gmx.at}\hspace{1.3cm} \\
  \AND
  Hormoz Shahrzad\footnotemark[2] \\
  UT Austin \& Cognizant AI Lab \\
  \texttt{hormoz@cognizant.com} \\
  \And
  Olivier Francon  \\
  Cognizant AI Lab \\  
  San Francisco, CA 94105, USA \\
  \texttt{olivier.francon@cognizant.com} \\
  \AND
  Babak Hodjat \\
  Cognizant AI Lab \\
  San Francisco, CA 94105, USA \\
  \texttt{babak.hodjat@cognizant.com} \\
  \And
  Risto Miikkulainen \\
  UT Austin \& Cognizant AI Lab \\
  \texttt{risto@cs.utexas.edu} \\
}

\setcitestyle{notesep={; }}
\iclrfinalcopy 
\begin{document}

\maketitle

\begin{abstract}
LLM pretraining is extremely costly; therefore, parameter-efficient LLM architectures have recently emerged as a compelling research direction. One such promising approach is to represent the parameters as orthonormal low-rank weight matrices. However, maintaining orthonormality during training is computationally expensive, making it impractical. This paper presents the TSVD (Truncated Singular Value Decomposition) framework which efficiently maintains orthonormality through QR decomposition and caching. Furthermore, a spectral energy heuristic is introduced to select the rank  of the resulting low-rank weight matrices. Empirical evaluations across model sizes show that TSVD matches or outperforms full-parameter baselines at a fraction of the compute cost. TSVD thus provides a scalable, computationally efficient foundation for LLM pretraining.
\end{abstract}

\section{Introduction}

The number of parameters in LLMs has grown exponentially in recent years, scaling from millions to trillions \citep{article, minaee2025largelanguagemodelssurvey, zhao2026surveylargelanguagemodels}. To pretrain such models, immense computational resources are needed \citep{hoffmann2022trainingcomputeoptimallargelanguage, kaplan2020scalinglawsneurallanguage}. To mitigate these costs, various optimization strategies have been developed. In particular, representing LLM weight matrices via low-rank approximations has gained prominence as an effective approach \citep{kamalakara2022exploringlowranktraining, khodak2022initializationregularizationfactorizedneural, savostianova2023robustlowranktrainingapproximate, sui2024elrtefficientlowranktraining, wei2024investigatinglowranktrainingtransformer}.

Representing LLM weight matrices as low-rank products drastically reduces the number of parameters and the computational overhead. However, this approach typically increases training and evaluation loss compared to full-parameter models \citep{zhao2024galorememoryefficientllmtraining, han2024sltrainsparsepluslowrank, wei2024investigatinglowranktrainingtransformer}. A potential solution is to make these matrices orthonormal, which alleviates vanishing and exploding gradients, reduces overfitting, enhances generalization, and leads to higher accuracy and faster convergence rates \citep[][see App.~\ref{app:tsvd_orthonormality} for details]{arjovsky2016unitaryevolutionrecurrentneural, bansal2018gainorthogonalityregularizationstraining, huang2017orthogonalweightnormalizationsolution}. However, strict orthonormality is computationally expensive. It can be approximated through orthonormality regularization \citep{yang2020learninglowrankdeepneural}, which introduces an extra hyperparameter to balance the penalty against the primary loss. Alternatively, low-rank initialization with periodic orthonormalization balances computational overhead and performance \citep{mo2025parameter}.

This paper proposes Truncated Singular Value Decomposition (TSVD) as an efficient low-rank representation for LLM weight matrices. Unlike existing low-rank methods, TSVD maintains strict orthonormality with minimal computational overhead. Preserving orthonormality yields lower training and evaluation loss than other low-rank approximations. Empirically, models with TSVD layers perform on par with or better than full-parameter models. Furthermore, their reduced memory and compute footprint make it possible to scale to larger LLM architectures with the same hardware.

Thus, this research closes the performance gap between compressed, low-rank models and their full-parameter counterparts. The main contributions are:

\begin{enumerate}
    \item An efficient three-fold strategy for maintaining orthonormality:
    \begin{enumerate}
    \item QR-based retraction: Replaces the expensive SVD steps required by standard Riemannian optimization with fast QR decomposition.
    \item Amortized computation: Deploys a caching mechanism that amortizes the remaining orthonormalization overhead across gradient accumulation steps.
    \item Rank selection: A spectral-energy heuristic for selecting different ranks for weight matrices within a transformer decoder layer prior to training, yielding better pretraining performance than uniform rank selection.
    \end{enumerate}
    \item By substantially reducing the parameter count, the TSVD framework reduces LLM pretraining FLOPs, lowering wall-clock time and GPU costs. Its lower memory footprint also enables training larger models on existing hardware or cheaper GPUs.
    \item A demonstration that TSVD achieves competitive or superior performance compared to dense, full-parameter models, despite significantly fewer parameters.    
\end{enumerate}

Thus, TSVD offers a practical framework for efficient pretraining, reducing computational costs while maximizing LLM performance within fixed compute budgets.

\section{Related Work}

This section reviews previous approaches to reducing LLM pretraining costs. They are categorized into two main groups: (1) parameter-preserving methods, which improve training efficiency without modifying model scale, and (2) parameter-reducing methods, which factorize or compress weight matrices directly. For each category, the key methodologies are discussed alongside their inherent limitations, followed by the subsequent techniques developed to address those shortcomings.

\textbf{Parameter-preserving methods}. GaLore \citep{zhao2024galorememoryefficientllmtraining} reduced the pretraining memory requirements without reducing the number of parameters. They projected the gradient matrix into a low-rank form, substantially reducing the memory cost for the optimizer states. However, during pretraining, the optimizer states only constitute a small proportion of the total memory, limiting the effectiveness of the approach \citep{shamshoum2025compactcompressedactivationsmemoryefficient}.

Low-Rank Adaptation \citep[LoRA;][]{hu2021loralowrankadaptationlarge} is a popular \textit{fine-tuning} method that efficiently adapts a pretrained LLM to a specific application or problem. LoRA freezes the weights of the pretrained model and adds two low-rank trainable matrices; during the forward pass, it multiplies the input by both the frozen weights and the product of the low-rank matrices. Inspired by LoRA, ReLoRA \citep{lialin2023relorahighranktraininglowrank} \textit{pretrains} an LLM using multiple low-rank updates. Instead of updating the full weights, ReLoRA trains low-rank matrices and periodically merges the low-rank updates into the main weights. It then reinitializes the low-rank matrices and repeats this process. Over time, multiple low-rank updates accumulate into a high-rank update, enabling efficient training without updating all parameters at once. However, ReLoRA requires a "warm start" full-parameter training for good performance and depends on an intricate learning-rate schedule. GaLore and ReLoRA successfully lower memory overhead for optimizer states and gradients but do not decrease the total number of trained parameters for even greater computational efficiency.

\textbf{Parameter-reducing methods}. Several studies investigated pretraining models via low-rank weight matrix representations, effectively reducing the number of trainable parameters \citep{kamalakara2022exploringlowranktraining, khodak2022initializationregularizationfactorizedneural, savostianova2023robustlowranktrainingapproximate, sui2024elrtefficientlowranktraining, wei2024investigatinglowranktrainingtransformer}. While this approach reduced pretraining memory and compute requirements, its training and evaluation loss was worse than that of full-parameter training \citep{zhao2024galorememoryefficientllmtraining, sui2024elrtefficientlowranktraining, wei2024investigatinglowranktrainingtransformer}. 

Incorporating orthonormality into low-rank representations serves as a strategy to mitigate the performance trade-offs typical of low-rank methods \citep{fernandezlopez2024fullrankmorelowrankweight, han2024sltrainsparsepluslowrank, li2025lostlowranksparsepretraining, yang2020learninglowrankdeepneural}. To enhance the performance of low-rank pretraining, LR-SMS \citep{fernandezlopez2024fullrankmorelowrankweight} proposed SVD initialization, wherein low-rank factors are initialized via SVD \citep{Eckart1936} to ensure initial orthonormality. SLTrain \citep{han2024sltrainsparsepluslowrank} and LOST \citep{li2025lostlowranksparsepretraining} decomposed each weight matrix into additive low-rank and sparse components. Both methods also utilized SVD initialization to provide a stable starting point, where the sparse component served as a complementary mechanism to recover information lost during compression by capturing salient features within the residual space. While LR-SMS, SLTrain, and LOST utilize SVD to initialize low-rank factors as orthonormal, this property is not preserved during training. Because the parameters remain unconstrained during optimization, they inevitably deviate from orthonormality. To address this limitation, SVD Training \citep{yang2020learninglowrankdeepneural} uses orthonormality regularizers to maintain orthonormality during training. It requires an additional hyperparameter to balance the regularization penalty against the primary loss objective.

The Low-Rank Riemannian Optimizer (LORO) \citep{mo2025parameter} treated weight matrices as elements of a fixed-rank Riemannian manifold. During each update, low-rank factors were jointly optimized following the Riemannian gradient, facilitating effective descent of the loss function. This retraction-based process involved: (1) projecting the Euclidean gradient onto the tangent space to update the weights, and (2) retracting the updated weights back onto the manifold via SVD. In contrast to SVD Training, LORO avoided explicit orthonormality regularizers. To maintain computational efficiency, the method bypassed the costly SVD retraction for the majority of iterations, performing the operation only every $K$ steps (e.g., $K=500$). Consequently, the weight matrices remain off the intended manifold for the bulk of the training process. StelLA \citep{li2026stellasubspacelearninglowrank} is conceptually similar to LORO; however, it is designed for parameter-efficient fine-tuning (PEFT) rather than pretraining. This shift introduces fundamental differences in key training dynamics, including gradient behavior, optimization stability, and parameter scaling.

Distinct from the weight-centric strategies mentioned above, COLA \citep{liu2025colacomputeefficientpretrainingllms} leverages the low-rank nature of LLM activations by replacing Multilayer Perceptron (MLP) and attention projection layers with low-rank autoencoders. Its memory-efficient variant, COLA-M, further reduces overhead by utilizing low-rank activations, while both frameworks preserve the original query, key, and value projection structures. Finally, to control the spectral norm of weight matrices and improve generalization, POET \citep{qiu2025reparameterizedllmtrainingorthogonal} and its more efficient variant, POET-X \citep{qiu2026poetxmemoryefficientllmtraining}, reparameterize $W$ as $\mathit{RWP}$. Both methods freeze $W$ and learn orthonormal matrices $R$ and $P$ using distinct approximations; however, the frozen weights must still be retained in memory.

In summary, low-rank weight matrix representations offer the most significant reductions in pretraining memory and computational overhead. The integration of orthonormality constraints helps mitigate the performance trade-offs typical of low-rank methods. However, due to high computational costs, previous techniques have not maintained orthonormality throughout the entire training duration. The TSVD method proposed in this paper ensures strict orthonormality at every iteration and, to the best of our knowledge, is the first approach to achieve this goal in a pretraining setting.

\section{The TSVD Framework}

In this section, SVD and TSVD are reviewed, and the TSVD layer, model, and its corresponding training methodology are introduced. Next, QR decomposition, caching, and rank selection methods are described. The TSVD algorithm is summarized in \ref{alg:tsvd_algorithm} and compared to the state of the art in \ref{app:comparison}.

\subsection{Singular Value Decomposition (SVD)}

SVD \citep{Eckart1936} is a fundamental matrix factorization method widely used in linear algebra, statistics, and machine learning. At its core, it expresses any matrix as a product of three simpler matrices. For any $m \times n$ matrix $W$, the SVD is
\begin{equation} \label{eq:svd}
W = U \Sigma V^T,
\end{equation}
where $U \in \mathbb{R}^{m \times m}$ comprises the left singular vectors (orthonormal matrix), $\Sigma \in \mathbb{R}^{m \times n}$ are singular values (diagonal matrix with non-negative entries), and $V \in \mathbb{R}^{n \times n}$ comprises the right singular vectors (orthonormal matrix). Thus,
SVD decomposes a linear transformation into three distinct geometric operations: $V^T$ rotates the input into a new coordinate system, $\Sigma$ scales the vector along orthogonal axes, and $U$ rotates the result into the output's coordinate system.

For any $m \times n$ matrix $W$, Truncated SVD (TSVD) is a low-rank approximation of SVD, defined as
\begin{equation} \label{eq:tsvd}
W \approx  U_k \Sigma_k V_k^T,
\end{equation}
where $U_k$  is the first $k$ columns of $U$, $\Sigma_k$ is the diagonal matrix with the top $k$ singular values, and $V_k$ is the first $k$ columns of $V$. Thus, the truncated version discards the smaller singular values and their corresponding directions.

\subsection{The TSVD Terminology}

The TSVD framework consists of three core abstractions. A \textbf{TSVD layer} approximates a full weight matrix $W \in \mathbb{R}^{m \times n}$ via truncated factorization $W \approx U_k \Sigma_k V_k^T$. Here, $U_k \in \mathbb{R}^{m \times k}$, $\Sigma_k \in \mathbb{R}^{k \times k}$, and $V_k^T \in \mathbb{R}^{k \times n}$ are initialized as unconstrained trainable parameters. During the forward pass, the layer enforces structural constraints: $U_k$ and $V_k$ are projected onto orthonormal matrices via QR decomposition \citep{chan1987rank}, while $\Sigma_k$ is constrained to a positive-valued diagonal vector. Constraining $U_k$ and $V_k$ to be orthonormal is what provides the stability benefits established in prior work on orthogonal and semi-orthogonal parameterizations, summarized in App.~\ref{app:tsvd_orthonormality}. Second, a \textbf{TSVD model} is a neural network that replaces standard dense weight matrices with TSVD layers. Third, \textbf{TSVD training} is the end-to-end training process of a TSVD model, wherein structural orthonormality constraints are maintained across every training iteration.

\subsection{QR Decomposition}

Low-rank approximation of $W$ reduces the number of parameters from $mn$ to $mk + k + kn = k(m+n+1) \approx k(m+n)$. For example, if $W$ is a $1000$ by $800$ matrix, and $k = 128$, the number of parameters reduces from $800,000$ to $128*(1000 + 800 +1) = 230,528$, i.e.\ by $\approx 71\%$. The benefits are twofold: (1) lower memory requirements for storing parameters, gradients, and optimizer states, and (2) reduced computational cost, as measured in FLOPs, because multiplying the input by a low-rank approximation of $W$ is less costly than multiplying the input by $W$ (see App.~\ref{app:full_vs_rank_k_flops} for details). The TSVD framework relies on QR decomposition in the forward pass to enforce orthonormality on unconstrained trainable parameters. The cost, in the number of FLOPs, for a matrix $\in \mathbb{R}^{m \times k}$ is
\begin{equation} \label{eq:qr_flops}
2mk^2 - \frac{2}{3}k^3.
\end{equation}

\subsection{Caching}

To mitigate the computational overhead of QR decomposition, the structure of gradient accumulation can be leveraged \citep{goodfellow2016deep}. While larger batch sizes make training of LLMs more stable, hardware constraints often require splitting an effective batch into smaller mini-batches. For example, an effective batch of 1024 is processed as four mini-batches of 256, with gradients accumulated before a single parameter update. The TSVD training includes a caching mechanism that computes the QR decomposition during the initial mini-batch and reuses the result for all subsequent mini-batches within the same update. This approach ensures that the expensive orthonormalization occurs only once per parameter update, effectively amortizing the computational cost while preserving the benefits of orthonormality. For ablation studies of these mechanism, see Sec.~\ref{sec:grad_accu}.

\subsection{Rank Selection}

Existing low-rank methods \citep{han2024sltrainsparsepluslowrank,mo2025parameter, liu2025colacomputeefficientpretrainingllms} typically apply a uniform rank across all layers and weight matrices, regardless of their functional roles. However, this choice may be suboptimal, because weight matrices in different components of a transformer decoder layer can require different ranks to perform well.

The TSVD framework employs the spectral energy of weight matrices of pretrained models as a heuristic for rank selection. First, SVD is applied to the weight matrices of the pretrained model, obtaining the diagonal matrix $\Sigma$ with singular values $\sigma_1, \sigma_2, \dots, \sigma_n$. The total energy of a full-rank matrix $W$ is then defined as the sum of these values. Second, to determine the rank for the low-rank approximation, the energy for a truncated set of ranks is calculated.  The \textit{compression rank}, $r_{\alpha}(W)$, to preserve at least $\alpha$ (e.g., $95\%$) of the total energy is then defined as 
\begin{equation} \label{eq:effective_rank}
    r_{\alpha}(W) \ = \ \text{min} \left\{ k \ \middle| \ \frac{\sum_{i=1}^{k} \sigma_i^2}{\sum_{j=1}^{n} \sigma_j^2} \geq \alpha \right\} \ .
\end{equation}

The value of $\alpha$ is selected based on the desired reduction in model's parameter count and is used to determine the compression rank of the weight matrices across different layers. Since compression rank generally varies across layers -- for example, MLP weight matrices typically exhibit higher compression ranks than attention weight matrices -- this approach supports variable rather than uniform rank selection. In the next section, we analyze the compression ranks across models of various sizes within an LLM family to determine their configurations.

\begin{algorithm}[t]
\caption{TSVD Training with Caching}
\label{alg:tsvd_algorithm}
\begin{algorithmic}[1]
\REQUIRE TSVD model; Spectral energy compression ranks $\{r_i\}_{i=1}^N$ of weight matrices
\FOR{each TSVD weight matrix $W_i$}
    \STATE Initialize $A_{U,i}\in\mathbb{R}^{m_i\times r_i}, A_{V,i}\in\mathbb{R}^{n_i \times r_i}, A_{\Sigma, i} \in\mathbb{R}^{r_i}$
\ENDFOR
\FOR{each optimizer step}
    \FOR{each mini-batch}
        \FOR{each TSVD weight matrix $W_i$}
            \IF{cache is populated}
                \STATE Read cached $(U_i,V_i,\Sigma_i)$
            \ELSE
                \STATE $U_i\gets\ \mathrm{qf}(A_{U_i})$
                \STATE $V_i\gets\ \mathrm{qf}(A_{V_i})$
                \STATE $\Sigma_i\gets\ \mathrm{diag}(\mathrm{softplus}({A_\Sigma}_i))$
                \STATE cache $(U_i,V_i,\Sigma_i)$
            \ENDIF
            \STATE run forward pass using $W_i := \frac{1}{\sqrt{r_i}}U_i\Sigma_i V_i^\top$
        \ENDFOR
        \STATE run backward pass, and gradient accumulation
    \ENDFOR
    \STATE update $\{A_{U,i},A_{V,i},A_{\Sigma,i} \}$ for $i = 1, \ldots, N$
    \STATE clear all TSVD caches
\ENDFOR
\STATE \emph{Note:} $\mathrm{qf}(\cdot)$ denotes the $Q$ factor of a reduced QR decomposition.
\end{algorithmic}
\end{algorithm}

\section{Configuring Compression Ranks} \label{sec:effective_rank_analysis}

Compression ranks are analyzed across the experimental suite of full-parameter Llama-style LLMs \citep{touvron2023llamaopenefficientfoundation}, ranging from 60M to 1.3B parameters. Compression rank (Eq.~\ref{eq:effective_rank}) is computed for each linear weight matrix within the transformer decoder layers. Those are the query, key, value and out projection within the attention mechanism and the gate, up and down projections in the MLP. The compression ranks for the models in the experimental suite are provided in Fig.~\ref{fig:effective_rank_main} in the Appendix. Across layers, the compression ranks of different projection types (query, key, etc.) are consistent, with the exception of the earliest layers. Motivated by this stability, compression ranks are averaged across layers for comparison across model scales.

Fig.~\ref{fig:llama_energy_sweep_a} illustrates the compression ranks across the different model sizes in the experimental suite. The main observation is that they scale approximately linearly with model size. Notably, the major architectural parameters (Tab.~\ref{tab:major_parameters}), that is, the number of transformer layers $n_{\text{layers}}$, the number of heads per layer $n_{\text{heads}}$, the hidden dimension $d_1$, and the MLP intermediate dimension $d_2$, scale uniformly across the models. When normalizing compression ranks (we chose the  compression rank of the out projection) of the model (Fig.~\ref{fig:llama_energy_sweep_b}), the relative proportions between ranks of different projections are nearly constant across model sizes. Overall, these results indicate that \textit{compression ranks measured on smaller models can be used to estimate those of larger models algorithmically}.

To assess the generality of the findings and highlight potential deviations, the analysis was extended to large, publicly available model families:
Llama-2 \citep{touvron2023llama}, Llama-3.1 \citep{dubey2024llama}, and Olmo 3 \citep{olmo2026olmo3}. Detailed results are provided in Appendix~\ref{sec:rank_analysis}. The results show that the compression ranks determined for small models can be used to predict the compression ranks at larger model sizes given the correct scaling of its hyperparameters.

For the pretraining experiments in Sec.~\ref{sec:pretraining}, the consistent relationships observed between the compression ranks of different functional components in Fig.~\ref{fig:llama_energy_sweep_a} guided the rank selection; E.g., the compression ranks for the value and output projections layers were $18\%$ larger than those for the query and key projections. The compression rank ratios were then maintained when scaling to larger target architectures in the rest of the experiments. That is, while absolute ranks are scaled according to a target parameter budget, the relative rank selections adhere to these relationships.

\begin{figure}[t]
\centering
\begin{subfigure}[b]{0.44\linewidth}
  \centering
  \hspace*{-4ex}\includegraphics[height=3.2cm]{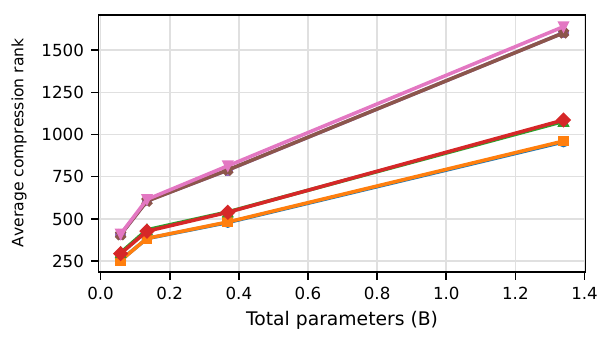}\\[-1ex]
  \caption{Unnormalized}
  \label{fig:llama_energy_sweep_a}
\end{subfigure}
\hfill
\begin{subfigure}[b]{0.54\linewidth}
  \centering
  \hspace*{5.1ex}\includegraphics[height=3.2cm]{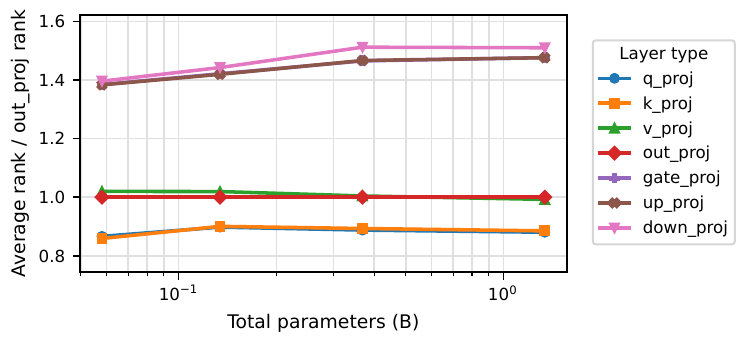}\\[-1ex]
  \caption{Normalized by out projection}
  \label{fig:llama_energy_sweep_b}
\end{subfigure}
\vspace*{-1ex}
\caption{Compression ranks measured across layers and model sizes. (a) Compression ranks of the different weight matrices (listed in the legend) in the transformer layers of the models in the experimental suite. The size of the model is shown in the $x$ axis; the ranks are averaged for each layer type within the same model. The ranks grow approximately linearly over the model size. (b) Compression ranks in Fig.~\ref{fig:llama_energy_sweep_a}  normalized by the size of the out projection for each model size.  The proportions remain nearly constant across model scales. These observations suggest that compression ranks measured on smaller models can be scaled to estimate those in larger models.}
\label{fig:llama_energy_sweep}
\end{figure}

\begin{table}[!t]
\centering
\setlength{\tabcolsep}{6pt}
\caption{Major architectural parameters of models in the experimental suite. Those are embedding dimension $d_1$, the hidden MLP dimension $d_2$, the number of layer $n_{\text{layers}}$ and the number of heads $n_{\text{heads}}$. Gray values denote scaling factors relative to the smallest model in the model family. Most architectural parameters scale proportionally with model size. As shown in Fig.~\ref{fig:llama_energy_sweep}, compression ranks scale consistently, allowing smaller models to estimate ranks for larger ones.}
\label{tab:major_parameters}
\fontsize{8pt}{9pt}\selectfont
\begin{tabular}{cccccccccccc}
\toprule
Model size / \textcolor{gray}{factor} &  & 60m   &  & 130m  & \textcolor{gray}{$\times$ 60m} &  & 350m  & \textcolor{gray}{$\times$ 60m} &  & 1.3B    & \textcolor{gray}{$\times$ 60m} \\ \midrule
$d_1$    &  & 512  &  & 768  & \textcolor{gray}{1.5}         &  & 1024  & \textcolor{gray}{2}            &  & 2048  & \textcolor{gray}{4}            \\
$d_2$       &  & 1376 &  & 2048 & \textcolor{gray}{1.5}         &  & 2736 & \textcolor{gray}{2}          &  & 5461 & \textcolor{gray}{4}            \\

$n_{\text{layers}}$     &  & 8    &  & 12    & \textcolor{gray}{1.5}         &  & 24    & \textcolor{gray}{3}          &  & 24    & \textcolor{gray}{3}          \\
$n_{\text{heads}}$   &  & 8    &  & 12    & \textcolor{gray}{1.5}         &  & 16    & \textcolor{gray}{2}            &  & 32    & \textcolor{gray}{4}            \\
\bottomrule
\end{tabular}
\end{table}

\section{Experiments}\label{sec:experiments}

This section details preliminary experiments establishing hyperparameters, followed by a performance comparison across model scales for TSVD, full-parameter, and low-rank architectures during pretraining and evaluation. Finally, ablation studies examine the impact of the QR decomposition, caching mechanism, and the role of variable rank selection.\footnote{The source code will be made publicly available on GitHub upon publication of the paper and is included as a ZIP archive in the supplementary material.}

\subsection{Hyperparameter Setup}

To tune hyperparameters across the TSVD experimental suite, the Maximal Update Parametrization \citep[$\mu$P;][]{yang2022tensorprogramsvtuning} framework is employed, which mitigates the prohibitive computational cost of optimization in large neural networks. By properly scaling initialization and learning rates, $\mu$P guarantees stable forward-pass activations and backward-pass gradient updates as model width increases. Crucially, this property enables zero-shot hyperparameter transfer, allowing optimal learning rates identified on small, computationally cheap models to transfer directly to significantly larger models without additional tuning. Accordingly, TSVD hyperparameter searches used the lightweight Llama-60M variant, requiring approximately 30 minutes of compute.

\begin{table}[b]
\caption{Pretraining results on C4 dataset, including validation perplexity (PPL), number of parameters in millions (Param), and memory usage in GiB (Mem) for parameters, gradients, and activations at 16-bit precision. $\downarrow$ indicates lower is better. The symbol $r$ is the rank for low-rank and CoLA models. K, Q, V, Out, and MLP are ranks for attention, out projection, and MLP layers. Results show that TSVD outperforms full-parameter and low-rank methods with lower computational cost.}
\label{1b-results-table}
\centering
\fontsize{8pt}{9pt}\selectfont
\setlength{\tabcolsep}{3pt}
\begin{tabular}{lcccccccccccc}
\toprule 
\multicolumn{1}{c}{Model (\# Tokens)} 
& \multicolumn{3}{c}{Llama-60M (1.1B)} 
& \multicolumn{3}{c}{Llama-130M (2.2B)} 
& \multicolumn{3}{c}{Llama-350M (6.4B)} 
& \multicolumn{3}{c}{Llama-1.3B (13.1B)} \\

\multicolumn{1}{c}{$r$}
& \multicolumn{3}{c}{128}
& \multicolumn{3}{c}{256}
& \multicolumn{3}{c}{256}
& \multicolumn{3}{c}{512} \\

\multicolumn{1}{c}{K, Q / V, Out / MLP}
& \multicolumn{3}{c}{96 / 112 / 160}
& \multicolumn{3}{c}{168 / 192 / 280}
& \multicolumn{3}{c}{168 / 192 / 280}
& \multicolumn{3}{c}{392 / 456 / 640} \\

\cmidrule(lr){1-1} \cmidrule(lr){2-4} \cmidrule(lr){5-7} \cmidrule(lr){8-10} \cmidrule(lr){11-13}
Method 
& PPL $\downarrow$ & Param & Mem $\downarrow$ 
& PPL $\downarrow$ & Param & Mem $\downarrow$
& PPL $\downarrow$ & Param & Mem $\downarrow$
& PPL $\downarrow$ & Param & Mem $\downarrow$ \\
\midrule
Full-parameter
& 16.02 & 58  & 0.43 
& 12.08 & 134  & 1.00 
& 9.20 & 368 & 2.74
& 7.68 & 1339 & 9.98 \\

Low-rank 
& 14.99 & \textbf{43}  & \textbf{0.32} 
& 11.73 & ~~94  & 0.70 
& 9.23 & 185 & 1.37 
& 7.72 & ~~\textbf{609} & \textbf{4.54} \\

CoLA 
& 14.84 & 46  & 0.34 
& 11.58 & 101 & 0.75 
& 9.23 & 223 & 1.66 
& 7.74 & ~~760 & 5.66 \\

TSVD 
& \textbf{14.40} & \textbf{43}  & \textbf{0.32} 
& \textbf{11.36} & ~~\textbf{91}  & \textbf{0.68}
& \textbf{9.14} & \textbf{177} & \textbf{1.32} 
& \textbf{7.56} & ~~644 & 4.80 \\

\bottomrule
\end{tabular}
\end{table}

\vspace*{-1ex}

\begin{table}[t]
\caption{Model perplexity ($\downarrow$) during training steps 50k to 150k for 7B full-parameter and TSVD models. TSVD is more parameter and memory efficient and closes the perplexity gap to the full parameter model as training progresses.}
\label{7b-results-table-perplexity}
\centering
\fontsize{8pt}{9pt}\selectfont
\begin{tabular}{lccccccc}
\toprule
Steps (\# Tokens) & 50k (6.5B) & 75k (9.8B) & 100k (13.1B) & 120k (15.7B) & 150k (19.6B) & &   \\ \cmidrule(lr){1-6}
Method & & & PPL $\downarrow$ & & &  Param & Mem $\downarrow$ \\
\midrule
Full-parameter & 7.35 & 6.84 & 6.60 & 6.51 & 6.38 & 6,738 & 50.20 \\
TSVD & 7.83 & 7.27 & 6.90 & 6.70 & 6.57 & 4,574 & 34.08 \\
\bottomrule
\end{tabular}
\end{table}

\subsection{Performance Evaluation}

This section provides a performance comparison across model scales for TSVD, full-parameter, and low-rank architectures during pretraining and evaluation.

\subsubsection{Pretraining Performance} \label{sec:pretraining}

A series of Llama-style LLMs \citep{touvron2023llamaopenefficientfoundation} with scales ranging from 60M to 7B parameters were pretrained in the language modeling task. The experimental framework strictly followed the benchmark protocols established in the literature \citep{zhao2024galorememoryefficientllmtraining,han2024sltrainsparsepluslowrank}. However, the benchmarks were ported to the HuggingFace ecosystem \citep{wolf-etal-2020-transformers}, using \texttt{accelerate} for launching distributed runs across GPUs. Four architectural configurations were evaluated: dense full-parameter models, models with low-rank approximations of linear and attention layers, models with CoLA-based linear layers, and the proposed TSVD approach. This enables comparison of TSVD with traditional and state-of-the-art reparameterization methods.

Models were trained on the Colossal Clean Crawled Corpus \citep[C4;][]{c4, dodge2021documentinglargewebtextcorpora}, a massive web-scale dataset derived from Common Crawl. C4 uses strict heuristic filtering to remove low-quality, repetitive, and offensive data. Comprising of 365M documents and 150B, C4 serves as a standard benchmark for pretraining prominent models like T5 \citep{dodge2021documentinglargewebtextcorpora}.

The results for models up to 1.3B are included in Tab.~\ref{1b-results-table}. Across all model scales, TSVD models consistently achieved lower perplexity than their full-parameter counterparts. TSVD models contained between $24\%$ and $51\%$ fewer parameters, yielding a corresponding reduction in the memory footprint for parameters, gradients, and optimizer states. Note that while the full-parameter, low-rank, and CoLA experiments were designed to replicate results reported in the literature, the modern Hugging Face implementation yielded significantly improved accuracies for them. These enhancements may be attributed to several factors, including modernized weight initialization schemes, refined tokenization processes, and the utilization of better optimization hyperparameters.

The scale-up properties of TSVD were demonstrated further on a Llama-7B model, which takes an average of 10 full days to pretrain on an NVIDIA H200 GPU node. The results are presented in Tab.~\ref{7b-results-table-perplexity}. Although the TSVD model’s perplexity is $\approx 3\%$ higher than that of the full-parameter model, it contains over 2B fewer parameters -- significantly reducing memory overhead for parameters, gradients, and optimizer states. Furthermore, as shown in Tab.~\ref{7b-results-table-perplexity}, the perplexity gap between the full-parameter and TSVD models narrows as training progresses, decreasing from 0.48 at 50k steps to 0.30 at 100k, and 0.19 at 150k. Given the observed trajectory, further training would likely narrow this performance gap even more. These results demonstrate that TSVD scales effectively to larger models, offering a robust and memory-efficient approach to large-scale pretraining.

\subsubsection{Benchmark Performance} \label{sec:benchmarks}

Having established that TSVD models perform well in pretraining task, It is important to assess pretrained model transfer to application benchmarks. To this end, the pretrained full-parameter and TSVD models were evaluated in key commonsense reasoning and text/context understanding tasks: HellaSwag \citep{zellers2019hellaswagmachinereallyfinish}, PIQA \citep{bisk2019piqareasoningphysicalcommonsense}, and LAMBADA \citep{paperno2016lambadadatasetwordprediction}. Detailed accuracy results are reported in Tab.~\ref{tab:downstream_results}. The experimental results demonstrate that the TSVD models consistently match or outperform full-parameter models across the entire suite of evaluation benchmarks, with the sole exception of the Llama-350M variant on the LAMBADA dataset. Consequently, aligning closely with its strong pretraining trajectory, TSVD maintains competitive generalization and robust performance when transferred to external evaluation tasks.

\begin{table}[!t]
\centering
\scriptsize
\caption{Accuracy across model scales for full-parameter and TSVD architectures on commonsense reasoning and text/context understanding benchmarks, HellaSwag, LAMBADA, and PIQA.}
\fontsize{8pt}{9pt}\selectfont
\setlength{\tabcolsep}{3.35pt}
\label{tab:downstream_results}
\begin{tabular}{lcccccccccccc}
\toprule
& \multicolumn{3}{c}{Llama-60M} & \multicolumn{3}{c}{Llama-130M} & \multicolumn{3}{c}{Llama-350M} & \multicolumn{3}{c}{Llama-1.3B} \\
\cmidrule(lr){2-4} \cmidrule(lr){5-7} \cmidrule(lr){8-10} \cmidrule(lr){11-13}
Method & HS & LAMB & PIQA & HS & LAMB & PIQA & HS & LAMB & PIQA & HS & LAMB & PIQA \\
\midrule
Full-parameter & 25.4\% & 12.6\% & 58.5\% & 26.0\% & 15.9\% & 60.5\% & 29.1\% & \textbf{27.8}\% & \textbf{65.2}\% & 32.7\% & \textbf{33.5}\% & 68.7\% \\
TSVD      & \textbf{25.9}\% & \textbf{15.5}\% & \textbf{60.4}\% & \textbf{27.0}\% & \textbf{17.9}\% & \textbf{63.2}\% & \textbf{29.8}\% & 27.4\% & 64.6\% & \textbf{33.5}\% & \textbf{33.5}\% & \textbf{68.9}\% \\
\bottomrule
\end{tabular}
\end{table}

\begin{figure}[t]
    \centering
    \begin{subfigure}[b]{0.32\linewidth}
        \centering
        \hspace*{5.1ex}\includegraphics[height=3.2cm]{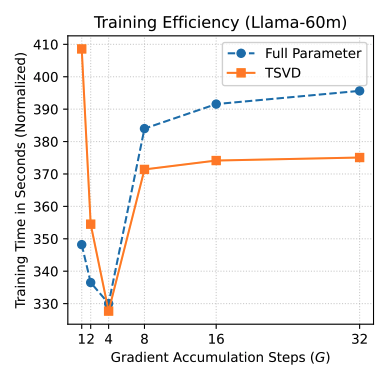}\\[-1ex]
        \caption{Llama-60M}
        \label{fig:llama_grad_accu_60m}
    \end{subfigure}
    \hfill
    \begin{subfigure}[b]{0.32\linewidth}
        \centering
        \hspace*{5.1ex}\includegraphics[height=3.2cm]{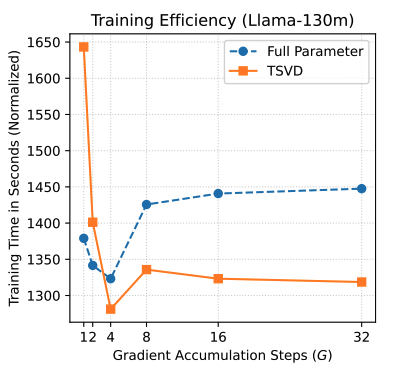}\\[-1ex]
        \caption{Llama-130M}
        \label{fig:llama_grad_accu_130m}
    \end{subfigure}
    \hfill
    \begin{subfigure}[b]{0.32\linewidth}
        \centering
        \hspace*{5.1ex}\includegraphics[height=3.2cm]{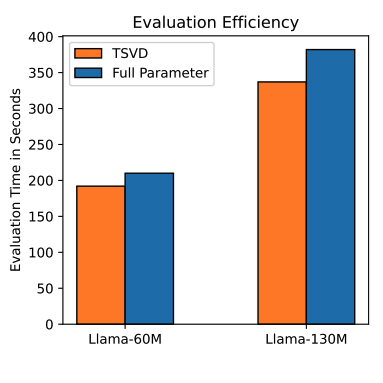}\\[-1ex]
        \caption{Normalized}
        \label{fig:llama_grad_accu_eval}
    \end{subfigure}
    \vspace*{-1ex}
    \caption{The impact of caching on \ref{fig:llama_grad_accu_60m} Llama-60M training, \ref{fig:llama_grad_accu_130m} Llama-130M training, and \ref{fig:llama_grad_accu_eval} Llama60m and Llama130m evaluation. As gradient accumulation steps increase, the TSVD model achieves greater gains in training efficiency compared to the full-parameter model, while maintaining consistent evaluation efficiency regardless of step size.}
    \label{fig:grad_accu}
\end{figure}

\subsection{Ablation Studies}

This section examines the impact of the QR decomposition, caching mechanism, and rank selection.

\subsubsection{Impact of QR decomposition}

The TSVD framework uses QR decomposition to maintain orthonormality, avoiding the computationally expensive SVD operation in Riemannian optimization. The respective computational costs are given by Equations~\ref{eq:qr_flops} and~\ref{eq:svd_flops}. Fig.~\ref{fig:qr_vs_svd} in Appendix plots the computational cost of SVD vs. QR decomposition for various ranks, showing that QR decomposition requires far fewer FLOPs.

\subsubsection{Impact of Caching Mechanism} \label{sec:grad_accu}

Gradient accumulation is a training technique that partitions an effective batch into smaller mini-batches, processing them sequentially while accumulating their gradients to execute a single parameter update once the full batch size is reached. The TSVD framework computes the QR decomposition during the initial mini-batch, caches the results, and reuses them across subsequent mini-batches, effectively amortizing orthonormalization costs. Consequently, an increased number of gradient accumulation steps leads to greater amortization of the QR decomposition overhead.

Full-parameter and TSVD Llama models (60M and 130M) were pretrained across $G \in \{1, 2, 4, 8, 16, 32\}$ gradient accumulation steps. Fig.~\ref{fig:grad_accu} shows that larger gradient accumulation steps and model sizes amplify TSVD’s training efficiency gains over full-parameter models. Note that at $G = 1$ (a single step), TSVD performs QR decomposition at every update without cache reuse, making it slower than the full-parameter model. However, as the number of gradient accumulation steps increases, the computational cost of QR decomposition is amortized across steps. The TSVD version thus outperforms the full-parameter version in wall-clock time, making it a more practical choice as the number of gradient accumulation steps increases.

\subsubsection{Impact of Compression Rank Selection}

While a uniform rank selection enforces a fixed rank across all layers and weight matrices irrespective of their functional roles, the TSVD framework computes the compression rank of each individual weight matrix based on its spectral energy, naturally accommodating variable rank selection across the architecture. To evaluate the impact of compression rank selections, four model scales were pretrained: Llama-60M, 130M, 350M, and 1.3B. At each scale, a TSVD model utilizing compression ranks (via spectral energy analysis) was compared against a uniform-rank baseline with fixed rank selections of 136, 240, 240, and 552, respectively. These uniform ranks were selected to match the total parameter count of the corresponding compression-rank models at each scale.

The experimental results are presented in Tab.~\ref{variable-rank-table}. The compression-rank TSVD models consistently outperform their uniform-rank counterparts, with the exception of the 60M scale. Furthermore, the performance advantages of compression rank selection widen as model scale increases. The results show that compression rank selection outperforms uniform rank selection.

\begin{table}[b]
\caption{Comparison of uniform vs.\ compression rank selections. In the 'Compression-rank' column, the triplet represents the ranks for the query/key projection, value/output projection, and MLP layer, respectively. Values in the 'Uniform-rank' column are applied uniformly across all layers. The compression-rank TSVD models outperform their uniform-rank counterparts, with the exception of the 60M model. Moreover, its performance advantage grows with model scale.}
\label{variable-rank-table}
\centering
\fontsize{8pt}{9pt}\selectfont
\setlength{\tabcolsep}{11pt}
\begin{tabular}{lcccc}
\toprule 
\multicolumn{1}{c}{Models} 
& \multicolumn{1}{c}{Uniform-rank}
& \multicolumn{1}{c}{Compression-rank}
& \multicolumn{1}{c}{Uniform-rank PPL}
& \multicolumn{1}{c}{Compression-rank PPL}\\
\midrule
Llama-60M
& 136
& ~~96, 112, 160
& \textbf{14.38}
& 14.40 \\

Llama-130M 
& 240
& 168, 192, 280
& 11.40
& \textbf{11.36} \\

Llama-350M
& 240 
& 168, 192, 280
& ~~9.19
& ~~\textbf{9.14} \\

Llama-1.3B
& 552 
& 392, 456, 640
& ~~7.62
& ~~\textbf{7.56} \\

\bottomrule
\end{tabular}
\end{table}

\section{Discussion and Future Work} \label{sec:discussion}

The TSVD experiments in this paper focused on dense models for two reasons: (1) to provide a fair comparison to prior low-rank pretraining methods, which were evaluated exclusively on dense architectures; and (2) to make sure that the approach scales to larger models. Having achieved these goals, a compelling next step is to extend it to Mixture-of-Experts (MoE) architectures \citep{shazeer2017outrageouslylargeneuralnetworks}, as the vast majority of state-of-the-art models rely on MoE architectures. Several other promising directions remain for future work. First, fine-tuning pretrained TSVD models on smaller, task-specific datasets is a natural next step. Second, investigating regularization strategies for $\Sigma$, to control the spectral norm of weight matrices and improve generalization. Third, exploring whether orthonormal matrices learned in smaller models can initialize corresponding matrices in larger models, while increasing the training efficiency.

\section{Conclusion}

This paper introduces TSVD, a novel framework designed to reduce the parameter count and computational overhead of LLM pretraining. To our knowledge, TSVD is the first approach to maintain strict weight orthonormality throughout training without incurring prohibitive overhead, enabled by QR decomposition, an optimized caching mechanism, and a rank selection heuristic. Empirical evaluations show that TSVD models match or outperform their full-parameter counterparts across pretraining metrics and evaluation benchmarks. TSVD reduces the FLOPs required for LLM pretraining, lowering wall-clock time and GPU cost. It also reduces memory requirements, enabling larger models on existing hardware or the use of lower-tier GPUs, further reducing pretraining costs. The TSVD framework provides a foundation for more efficient pretraining of LLMs.

\bibliography{iclr2027_conference}

@article{article,
         author = {Gu, Jiarui},
         year = {2025},
          month = {11},
          pages = {45-56},
          title = {From GPT to LLaMA: Tracing the Growth of Large Language Models},
          volume = {142},
          journal = {Theoretical and Natural Science},
          doi = {10.54254/2753-8818/2025.DL29003}
}

@misc{minaee2025largelanguagemodelssurvey,
      title={Large Language Models: A Survey}, 
      author={Shervin Minaee and Tomas Mikolov and Narjes Nikzad and Meysam Chenaghlu and Richard Socher and Xavier Amatriain and Jianfeng Gao},
      year={2025},
      eprint={2402.06196},
      archivePrefix={arXiv},
      primaryClass={cs.CL},
      url={https://arxiv.org/abs/2402.06196}, 
}

@misc{zhao2026surveylargelanguagemodels,
      title={A Survey of Large Language Models}, 
      author={Wayne Xin Zhao and Kun Zhou and Junyi Li and Tianyi Tang and Xiaolei Wang and Yupeng Hou and Yingqian Min and Beichen Zhang and Junjie Zhang and Zican Dong and Yifan Du and Chen Yang and Yushuo Chen and Zhipeng Chen and Jinhao Jiang and Ruiyang Ren and Yifan Li and Xinyu Tang and Zikang Liu and Peiyu Liu and Jian-Yun Nie and Ji-Rong Wen},
      year={2026},
      eprint={2303.18223},
      archivePrefix={arXiv},
      primaryClass={cs.CL},
      url={https://arxiv.org/abs/2303.18223}, 
}

@misc{hoffmann2022trainingcomputeoptimallargelanguage,
      title={Training Compute-Optimal Large Language Models}, 
      author={Jordan Hoffmann and Sebastian Borgeaud and Arthur Mensch and Elena Buchatskaya and Trevor Cai and Eliza Rutherford and Diego de Las Casas and Lisa Anne Hendricks and Johannes Welbl and Aidan Clark and Tom Hennigan and Eric Noland and Katie Millican and George van den Driessche and Bogdan Damoc and Aurelia Guy and Simon Osindero and Karen Simonyan and Erich Elsen and Jack W. Rae and Oriol Vinyals and Laurent Sifre},
      year={2022},
      eprint={2203.15556},
      archivePrefix={arXiv},
      primaryClass={cs.CL},
      url={https://arxiv.org/abs/2203.15556}, 
}

@misc{kaplan2020scalinglawsneurallanguage,
      title={Scaling Laws for Neural Language Models}, 
      author={Jared Kaplan and Sam McCandlish and Tom Henighan and Tom B. Brown and Benjamin Chess and Rewon Child and Scott Gray and Alec Radford and Jeffrey Wu and Dario Amodei},
      year={2020},
      eprint={2001.08361},
      archivePrefix={arXiv},
      primaryClass={cs.LG},
      url={https://arxiv.org/abs/2001.08361}, 
}

@misc{kamalakara2022exploringlowranktraining,
      title={Exploring Low Rank Training of Deep Neural Networks}, 
      author={Siddhartha Rao Kamalakara and Acyr Locatelli and Bharat Venkitesh and Jimmy Ba and Yarin Gal and Aidan N. Gomez},
      year={2022},
      eprint={2209.13569},
      archivePrefix={arXiv},
      primaryClass={cs.LG},
      url={https://arxiv.org/abs/2209.13569}, 
}

@misc{khodak2022initializationregularizationfactorizedneural,
      title={Initialization and Regularization of Factorized Neural Layers}, 
      author={Mikhail Khodak and Neil Tenenholtz and Lester Mackey and Nicolò Fusi},
      year={2022},
      eprint={2105.01029},
      archivePrefix={arXiv},
      primaryClass={stat.ML},
      url={https://arxiv.org/abs/2105.01029}, 
}

@misc{savostianova2023robustlowranktrainingapproximate,
      title={Robust low-rank training via approximate orthonormal constraints}, 
      author={Dayana Savostianova and Emanuele Zangrando and Gianluca Ceruti and Francesco Tudisco},
      year={2023},
      eprint={2306.01485},
      archivePrefix={arXiv},
      primaryClass={cs.LG},
      url={https://arxiv.org/abs/2306.01485}, 
}

@misc{sui2024elrtefficientlowranktraining,
      title={ELRT: Efficient Low-Rank Training for Compact Convolutional Neural Networks}, 
      author={Yang Sui and Miao Yin and Yu Gong and Jinqi Xiao and Huy Phan and Bo Yuan},
      year={2024},
      eprint={2401.10341},
      archivePrefix={arXiv},
      primaryClass={cs.CV},
      url={https://arxiv.org/abs/2401.10341}, 
}

@misc{wei2024investigatinglowranktrainingtransformer,
      title={Investigating Low-Rank Training in Transformer Language Models: Efficiency and Scaling Analysis}, 
      author={Xiuying Wei and Skander Moalla and Razvan Pascanu and Caglar Gulcehre},
      year={2024},
      eprint={2407.09835},
      archivePrefix={arXiv},
      primaryClass={cs.CL},
      url={https://arxiv.org/abs/2407.09835}, 
}

@misc{zhao2024galorememoryefficientllmtraining,
      title={GaLore: Memory-Efficient LLM Training by Gradient Low-Rank Projection}, 
      author={Jiawei Zhao and Zhenyu Zhang and Beidi Chen and Zhangyang Wang and Anima Anandkumar and Yuandong Tian},
      year={2024},
      eprint={2403.03507},
      archivePrefix={arXiv},
      primaryClass={cs.LG},
      url={https://arxiv.org/abs/2403.03507}, 
}

@misc{arjovsky2016unitaryevolutionrecurrentneural,
      title={Unitary Evolution Recurrent Neural Networks}, 
      author={Martin Arjovsky and Amar Shah and Yoshua Bengio},
      year={2016},
      eprint={1511.06464},
      archivePrefix={arXiv},
      primaryClass={cs.LG},
      url={https://arxiv.org/abs/1511.06464}, 
}

@misc{bansal2018gainorthogonalityregularizationstraining,
      title={Can We Gain More from Orthogonality Regularizations in Training Deep CNNs?}, 
      author={Nitin Bansal and Xiaohan Chen and Zhangyang Wang},
      year={2018},
      eprint={1810.09102},
      archivePrefix={arXiv},
      primaryClass={cs.LG},
      url={https://arxiv.org/abs/1810.09102}, 
}

@misc{yang2020learninglowrankdeepneural,
      title={Learning Low-rank Deep Neural Networks via Singular Vector Orthogonality Regularization and Singular Value Sparsification}, 
      author={Huanrui Yang and Minxue Tang and Wei Wen and Feng Yan and Daniel Hu and Ang Li and Hai Li and Yiran Chen},
      year={2020},
      eprint={2004.09031},
      archivePrefix={arXiv},
      primaryClass={cs.LG},
      url={https://arxiv.org/abs/2004.09031}, 
}

@inproceedings{mo2025parameter,
               author = {Mo, Zhanfeng and Huang, Long-Kai and Pan, Sinno Jialin},
               booktitle = {International Conference on Learning Representations},
               editor = {Y. Yue and A. Garg and N. Peng and F. Sha and R. Yu},
               pages = {7907--7936},
               title = {Parameter and Memory Efficient Pretraining via Low-rank Riemannian Optimization},
               url = {https://proceedings.iclr.cc/paper_files/paper/2025/file/16fafdc5e8ab22780873d607ce4387a8-Paper-Conference.pdf},
               volume = {2025},
               year = {2025}
}

@misc{shamshoum2025compactcompressedactivationsmemoryefficient,
      title={CompAct: Compressed Activations for Memory-Efficient LLM Training}, 
      author={Yara Shamshoum and Nitzan Hodos and Yuval Sieradzki and Assaf Schuster},
      year={2025},
      eprint={2410.15352},
      archivePrefix={arXiv},
      primaryClass={cs.LG},
      url={https://arxiv.org/abs/2410.15352}, 
}

@misc{hu2021loralowrankadaptationlarge,
      title={LoRA: Low-Rank Adaptation of Large Language Models}, 
      author={Edward J. Hu and Yelong Shen and Phillip Wallis and Zeyuan Allen-Zhu and Yuanzhi Li and Shean Wang and Lu Wang and Weizhu Chen},
      year={2021},
      eprint={2106.09685},
      archivePrefix={arXiv},
      primaryClass={cs.CL},
      url={https://arxiv.org/abs/2106.09685}, 
}

@misc{lialin2023relorahighranktraininglowrank,
      title={ReLoRA: High-Rank Training Through Low-Rank Updates}, 
      author={Vladislav Lialin and Namrata Shivagunde and Sherin Muckatira and Anna Rumshisky},
      year={2023},
      eprint={2307.05695},
      archivePrefix={arXiv},
      primaryClass={cs.CL},
      url={https://arxiv.org/abs/2307.05695}, 
}

@misc{fernandezlopez2024fullrankmorelowrankweight,
      title={Full-Rank No More: Low-Rank Weight Training for Modern Speech Recognition Models}, 
      author={Adriana Fernandez-Lopez and Shiwei Liu and Lu Yin and Stavros Petridis and Maja Pantic},
      year={2024},
      eprint={2410.07771},
      archivePrefix={arXiv},
      primaryClass={cs.SD},
      url={https://arxiv.org/abs/2410.07771}, 
}

@misc{han2024sltrainsparsepluslowrank,
      title={SLTrain: a sparse plus low-rank approach for parameter and memory efficient pretraining}, 
      author={Andi Han and Jiaxiang Li and Wei Huang and Mingyi Hong and Akiko Takeda and Pratik Jawanpuria and Bamdev Mishra},
      year={2024},
      eprint={2406.02214},
      archivePrefix={arXiv},
      primaryClass={cs.LG},
      url={https://arxiv.org/abs/2406.02214}, 
}

@misc{li2025lostlowranksparsepretraining,
      title={LOST: Low-rank and Sparse Pre-training for Large Language Models}, 
      author={Jiaxi Li and Lu Yin and Li Shen and Jinjin Xu and Liwu Xu and Tianjin Huang and Wenwu Wang and Shiwei Liu and Xilu Wang},
      year={2025},
      eprint={2508.02668},
      archivePrefix={arXiv},
      primaryClass={cs.LG},
      url={https://arxiv.org/abs/2508.02668}, 
}

@Article{Eckart1936,
         author={Eckart, Carl and Young, Gale},
         title={The approximation of one matrix by another of lower rank},
         journal={Psychometrika},
         year={1936},
         month={Sep},
         day={01},
         volume={1},
         number={3},
         pages={211-218},
         issn={1860-0980},
         doi={10.1007/BF02288367},
         url={https://doi.org/10.1007/BF02288367}
}

@misc{liu2025colacomputeefficientpretrainingllms,
      title={CoLA: Compute-Efficient Pre-Training of LLMs via Low-Rank Activation}, 
      author={Ziyue Liu and Ruijie Zhang and Zhengyang Wang and Mingsong Yan and Zi Yang and Paul Hovland and Bogdan Nicolae and Franck Cappello and Sui Tang and Zheng Zhang},
      year={2025},
      eprint={2502.10940},
      archivePrefix={arXiv},
      primaryClass={cs.LG},
      url={https://arxiv.org/abs/2502.10940}, 
}

@article{goodfellow2016deep,
  title={Deep feedforward networks},
  author={Goodfellow, Ian and Bengio, Yoshua and Courville, Aaron},
  journal={Deep learning},
  volume={1},
  pages={161--217},
  year={2016},
  publisher={MIT press Cambridge, MA, USA}
}

@book{gallot1990riemannian,
  title={Riemannian geometry},
  author={Gallot, Sylvestre and Hulin, Dominique and Lafontaine, Jacques and others},
  volume={2},
  year={1990},
  publisher={Springer}
}

@book{absil2008optimization,
      title={Optimization algorithms on matrix manifolds},
      author={Absil, P-A and Mahony, Robert and Sepulchre, Rodolphe},
      year={2008},
      publisher={Princeton University Press}
}

@book{boumal2023introduction,
      title={An introduction to optimization on smooth manifolds},
      author={Boumal, Nicolas},
      year={2023},
      publisher={Cambridge University Press}
}

@inproceedings{wolf-etal-2020-transformers,
    title = "Transformers: State-of-the-Art Natural Language Processing",
    author = "Wolf, Thomas  and
      Debut, Lysandre  and
      Sanh, Victor  and
      Chaumond, Julien  and
      Delangue, Clement  and
      Moi, Anthony  and
      Cistac, Pierric  and
      Rault, Tim  and
      Louf, Remi  and
      Funtowicz, Morgan  and
      Davison, Joe  and
      Shleifer, Sam  and
      von Platen, Patrick  and
      Ma, Clara  and
      Jernite, Yacine  and
      Plu, Julien  and
      Xu, Canwen  and
      Le Scao, Teven  and
      Gugger, Sylvain  and
      Drame, Mariama  and
      Lhoest, Quentin  and
      Rush, Alexander",
    editor = "Liu, Qun  and Schlangen, David",
    booktitle = "Proceedings of the 2020 Conference on Empirical Methods in Natural Language Processing: System Demonstrations",
    month = oct,
    year = "2020",
    address = "Online",
    publisher = "Association for Computational Linguistics",
    url = "https://aclanthology.org/2020.emnlp-demos.6/",
    doi = "10.18653/v1/2020.emnlp-demos.6",
    pages = "38--45"
}

@article{c4,
         author = {Colin Raffel and Noam Shazeer and Adam Roberts and Katherine Lee and Sharan Narang and Michael Matena and Yanqi Zhou and Wei Li and Peter J. Liu},
         title  = {Exploring the Limits of Transfer Learning with a Unified Text-to-Text Transformer},
         journal = {CoRR},
         volume = {abs/1910.10683},
         year = {2019},
         url = {http://arxiv.org/abs/1910.10683},
         eprinttype = {arXiv},
         eprint = {1910.10683},
         bibsource = {dblp computer science bibliography, https://dblp.org}
}

@misc{dodge2021documentinglargewebtextcorpora,
      title={Documenting Large Webtext Corpora: A Case Study on the Colossal Clean Crawled Corpus}, 
      author={Jesse Dodge and Maarten Sap and Ana Marasović and William Agnew and Gabriel Ilharco and Dirk Groeneveld and Margaret Mitchell and Matt Gardner},
      year={2021},
      eprint={2104.08758},
      archivePrefix={arXiv},
      primaryClass={cs.CL},
      url={https://arxiv.org/abs/2104.08758}, 
}

@misc{zellers2019hellaswagmachinereallyfinish,
      title={HellaSwag: Can a Machine Really Finish Your Sentence?}, 
      author={Rowan Zellers and Ari Holtzman and Yonatan Bisk and Ali Farhadi and Yejin Choi},
      year={2019},
      eprint={1905.07830},
      archivePrefix={arXiv},
      primaryClass={cs.CL},
      url={https://arxiv.org/abs/1905.07830}, 
}

@misc{paperno2016lambadadatasetwordprediction,
      title={The LAMBADA dataset: Word prediction requiring a broad discourse context}, 
      author={Denis Paperno and Germán Kruszewski and Angeliki Lazaridou and Quan Ngoc Pham and Raffaella Bernardi and Sandro Pezzelle and Marco Baroni and Gemma Boleda and Raquel Fernández},
      year={2016},
      eprint={1606.06031},
      archivePrefix={arXiv},
      primaryClass={cs.CL},
      url={https://arxiv.org/abs/1606.06031}, 
}

@misc{bisk2019piqareasoningphysicalcommonsense,
      title={PIQA: Reasoning about Physical Commonsense in Natural Language}, 
      author={Yonatan Bisk and Rowan Zellers and Ronan Le Bras and Jianfeng Gao and Yejin Choi},
      year={2019},
      eprint={1911.11641},
      archivePrefix={arXiv},
      primaryClass={cs.CL},
      url={https://arxiv.org/abs/1911.11641}, 
}

@misc{shazeer2017outrageouslylargeneuralnetworks,
      title={Outrageously Large Neural Networks: The Sparsely-Gated Mixture-of-Experts Layer}, 
      author={Noam Shazeer and Azalia Mirhoseini and Krzysztof Maziarz and Andy Davis and Quoc Le and Geoffrey Hinton and Jeff Dean},
      year={2017},
      eprint={1701.06538},
      archivePrefix={arXiv},
      primaryClass={cs.LG},
      url={https://arxiv.org/abs/1701.06538}, 
}

@article{touvron2023llama,
  title={Llama 2: Open foundation and fine-tuned chat models},
  author={Touvron, Hugo and Martin, Louis and Stone, Kevin and Albert, Peter and Almahairi, Amjad and Babaei, Yasmine and Bashlykov, Nikolay and Batra, Soumya and Bhargava, Prajjwal and Bhosale, Shruti and others},
  journal={arXiv preprint arXiv:2307.09288},
  year={2023}
}

@article{dubey2024llama,
      title={The Llama 3 Herd of Models}, 
      author={Abhimanyu Dubey and Abhinav Jauhri and Abhinav Pandey and Abhishek Kadian and Ahmad Al-Dahle and Aiesha Letman and Akhil Mathur and Alan Schelten and Amy Yang and Angela Fan and Anirudh Goyal and Anthony Hartshorn and Aobo Yang and Archi Mitra and Archie Sravankumar and Artem Korenev and Arthur Hinsvark and Arun Rao and Aston Zhang and Aurelien Rodriguez et al.},
      year={2024},
      volume={2407.21783},
      journal={arXiv},
      primaryClass={cs.AI},
}

@misc{olmo2026olmo3,
      title={Olmo 3}, 
      author={Team Olmo and : and Allyson Ettinger and Amanda Bertsch and Bailey Kuehl and David Graham and David Heineman and Dirk Groeneveld and Faeze Brahman and Finbarr Timbers and Hamish Ivison and Jacob Morrison and Jake Poznanski and Kyle Lo and Luca Soldaini and Matt Jordan and Mayee Chen and Michael Noukhovitch and Nathan Lambert and Pete Walsh and Pradeep Dasigi and Robert Berry and Saumya Malik and Saurabh Shah and Scott Geng and Shane Arora and Shashank Gupta and Taira Anderson and Teng Xiao and Tyler Murray and Tyler Romero and Victoria Graf and Akari Asai and Akshita Bhagia and Alexander Wettig and Alisa Liu and Aman Rangapur and Chloe Anastasiades and Costa Huang and Dustin Schwenk and Harsh Trivedi and Ian Magnusson and Jaron Lochner and Jiacheng Liu and Lester James V. Miranda and Maarten Sap and Malia Morgan and Michael Schmitz and Michal Guerquin and Michael Wilson and Regan Huff and Ronan Le Bras and Rui Xin and Rulin Shao and Sam Skjonsberg and Shannon Zejiang Shen and Shuyue Stella Li and Tucker Wilde and Valentina Pyatkin and Will Merrill and Yapei Chang and Yuling Gu and Zhiyuan Zeng and Ashish Sabharwal and Luke Zettlemoyer and Pang Wei Koh and Ali Farhadi and Noah A. Smith and Hannaneh Hajishirzi},
      year={2026},
      eprint={2512.13961},
      archivePrefix={arXiv},
      primaryClass={cs.CL},
      url={https://arxiv.org/abs/2512.13961}, 
}

@misc{touvron2023llamaopenefficientfoundation,
      title={LLaMA: Open and Efficient Foundation Language Models}, 
      author={Hugo Touvron and Thibaut Lavril and Gautier Izacard and Xavier Martinet and Marie-Anne Lachaux and Timothée Lacroix and Baptiste Rozière and Naman Goyal and Eric Hambro and Faisal Azhar and Aurelien Rodriguez and Armand Joulin and Edouard Grave and Guillaume Lample},
      year={2023},
      eprint={2302.13971},
      archivePrefix={arXiv},
      primaryClass={cs.CL},
      url={https://arxiv.org/abs/2302.13971}, 
}

@misc{yang2022tensorprogramsvtuning,
      title={Tensor Programs V: Tuning Large Neural Networks via Zero-Shot Hyperparameter Transfer}, 
      author={Greg Yang and Edward J. Hu and Igor Babuschkin and Szymon Sidor and Xiaodong Liu and David Farhi and Nick Ryder and Jakub Pachocki and Weizhu Chen and Jianfeng Gao},
      year={2022},
      eprint={2203.03466},
      archivePrefix={arXiv},
      primaryClass={cs.LG},
      url={https://arxiv.org/abs/2203.03466}, 
}

@misc{li2026stellasubspacelearninglowrank,
      title={StelLA: Subspace Learning in Low-rank Adaptation using Stiefel Manifold}, 
      author={Zhizhong Li and Sina Sajadmanesh and Jingtao Li and Lingjuan Lyu},
      year={2026},
      eprint={2510.01938},
      archivePrefix={arXiv},
      primaryClass={cs.LG},
      url={https://arxiv.org/abs/2510.01938}, 
}

@inproceedings{ainslie-etal-2023-gqa,
    title = "{GQA}: Training Generalized Multi-Query Transformer Models from Multi-Head Checkpoints",
    author = "Ainslie, Joshua  and
      Lee-Thorp, James  and
      de Jong, Michiel  and
      Zemlyanskiy, Yury  and
      Lebron, Federico  and
      Sanghai, Sumit",
    booktitle = "Proceedings of the 2023 Conference on Empirical Methods in Natural Language Processing",
    year = "2023",
    publisher = "Association for Computational Linguistics",
    pages = "4895--4901"
}

@misc{huang2017orthogonalweightnormalizationsolution,
      title={Orthogonal Weight Normalization: Solution to Optimization over Multiple Dependent Stiefel Manifolds in Deep Neural Networks}, 
      author={Lei Huang and Xianglong Liu and Bo Lang and Adams Wei Yu and Yongliang Wang and Bo Li},
      year={2017},
      eprint={1709.06079},
      archivePrefix={arXiv},
      primaryClass={cs.LG},
      url={https://arxiv.org/abs/1709.06079}, 
}

@misc{qiu2025reparameterizedllmtrainingorthogonal,
      title={Reparameterized LLM Training via Orthogonal Equivalence Transformation}, 
      author={Zeju Qiu and Simon Buchholz and Tim Z. Xiao and Maximilian Dax and Bernhard Schölkopf and Weiyang Liu},
      year={2025},
      eprint={2506.08001},
      archivePrefix={arXiv},
      primaryClass={cs.LG},
      url={https://arxiv.org/abs/2506.08001}, 
}

@misc{qiu2026poetxmemoryefficientllmtraining,
      title={POET-X: Memory-efficient LLM Training by Scaling Orthogonal Transformation}, 
      author={Zeju Qiu and Lixin Liu and Adrian Weller and Han Shi and Weiyang Liu},
      year={2026},
      eprint={2603.05500},
      archivePrefix={arXiv},
      primaryClass={cs.LG},
      url={https://arxiv.org/abs/2603.05500}, 
}

@article{chan1987rank,
  title={Rank revealing {QR} factorizations},
  author={Chan, Tony F.},
  journal={Linear Algebra and its Applications},
  volume={88-89},
  pages={67--82},
  year={1987},
  publisher={Elsevier},
  doi={10.1016/0024-3795(87)90103-0},
  url={https://doi.org/10.1016/0024-3795(87)90103-0}
}
\bibliographystyle{iclr2027_conference}

\appendix

\section{Why Orthonormality Helps}
\label{app:tsvd_orthonormality}

Consider a linear layer $W \in \mathbb{R}^{m\times n}$ parameterized in TSVD as $W = \tfrac{1}{\sqrt r} U \Sigma V^\top$ with $U^\top U = I_r$, $V^\top V = I_r$, and $\Sigma = \mathrm{diag}(\sigma_1,\dots,\sigma_r)$, $\sigma_i>0$. $U$ and $V$ are semi-orthogonal, i.e., points on the Stiefel manifolds $\mathrm{St}(m,r)$ and $\mathrm{St}(n,r)$ (App.~\ref{app:comparison}); all spectral scaling is carried by $\Sigma$, and $\|W\|_2 = \sigma_{\max}/\sqrt r$, where $\sigma_{\max} = \max_i \sigma_i$ and $\sigma_{\min}=\min_i \sigma_i$.

That constraining a linear map to be (semi-)orthogonal controls signal and gradient magnitude is a standard result, not specific to TSVD.

\begin{proposition}[Signal and gradient stability in orthonormal TSVD]
\label{prop:tsvd_stability}
For any input $x\in\mathbb{R}^n$ and any back-propagated signal $\delta\in\mathbb{R}^m$,
\[
\|Wx\|_2 \le \frac{\sigma_{\max}}{\sqrt r}\|x\|_2,
\qquad
\|W^\top\delta\|_2 \le \frac{\sigma_{\max}}{\sqrt r}\|\delta\|_2,
\]
with equality when $x \in \mathrm{span}(V)$ or $\delta \in \mathrm{span}(U)$, respectively, and a matching lower bound of $\sigma_{\min}/\sqrt r$ on that subspace.
\end{proposition}

This is a direct consequence of operator-norm submultiplicativity for (semi-)orthogonal maps. \citet{arjovsky2016unitaryevolutionrecurrentneural} proved the same bound for square orthogonal/unitary matrices, showing that it lets a recurrent network avoid vanishing and exploding gradients by construction rather than by clipping; \citet{huang2017orthogonalweightnormalizationsolution} extended the argument to rectangular, semi-orthogonal weight matrices in feedforward and convolutional networks, further showing that an orthogonal weight matrix preserves the mean and covariance of its input's distribution and, by fixing $r(r+1)/2$ degrees of freedom per factor, regularizes the network relative to an unconstrained factorization.
The proof of Proposition~\ref{prop:tsvd_stability} is identical in structure to \citeauthor{huang2017orthogonalweightnormalizationsolution}'s proof (and omitted here). All amplification and attenuation in TSVD is therefore due to $\Sigma$ alone; keeping $U,V$ orthonormal for the full training run, rather than only at initialization, is what makes this bound hold throughout training rather than just at the start.

A common alternative is a soft penalty such as $\lambda\|W^\top W - I\|_F^2$ \citep{bansal2018gainorthogonalityregularizationstraining}, which only approximately enforces the constraint and trades off against the task loss. The exact alternative, Riemannian optimization on the Stiefel manifold, is stable but conventionally requires an SVD-based retraction at every step (App.~\ref{app:comparison}). TSVD instead reparameterizes $U_k=\mathrm{qf}(A_{U_k})$ and $V_k=\mathrm{qf}(A_{V_k})$ via QR, the same reparameterize-rather-than-penalize strategy as \citet{huang2017orthogonalweightnormalizationsolution}'s orthogonal weight normalization, at the lower cost of a QR decomposition in place of an eigendecomposition or SVD.

An unconstrained low-rank factorization $W=AB^\top$ has no such guarantee: for any invertible $c$, $AB^\top = (cA)(c^{-1}B)^\top$ represents the identical $W$, so the individual factor gradients can grow and shrink without bound even while $W$'s own gradient does not. Fixing $U^\top U = V^\top V = I_r$ removes this ambiguity, since the norm of each factor is fixed by construction.

Recent work on LLM pretraining specifically reaches the same conclusion: \citet{qiu2025reparameterizedllmtrainingorthogonal} show that reparameterizing weight matrices through orthogonal transformations gives provable control over their singular-value spectrum, tied to standard spectral-complexity generalization bounds. TSVD's use of orthonormal $U$ and $V$ is best read in this light: an established, well-conditioned parameterization whose bounded spectral norm ($\sigma_{\max}/\sqrt r$, Proposition~\ref{prop:tsvd_stability}) carries the same generalization benefit, now validated at full LLM-pretraining scale as well, combined with rank truncation for parameter savings.

\section{Comparison to Orthonormal Representations}
\label{app:comparison}

Existing approaches to maintaining orthonormality fall into three categories: (1) initializing weight matrices as orthonormal, and not compensating for the inevitable deviation from orthonormality during unconstrained optimization \citep{fernandezlopez2024fullrankmorelowrankweight}; (2) employing soft regularizers \citep{yang2020learninglowrankdeepneural}, which introduce additional hyperparameters for encouraging orthonormality; and (3) periodically projecting weights back onto the orthonormal manifold \citep{mo2025parameter}. The latter approach is most closely related to the TSVD method and is discussed in this subsection.

The LORO method \citep{mo2025parameter} utilizes Riemannian optimization to maintain weight constraints. As the mathematical context, a manifold \citep{gallot1990riemannian} is a topological space that is locally flat but could be curved globally, such as a plane, sphere, torus, or saddle-shaped surface. A Riemannian manifold \citep{gallot1990riemannian} provides a way to measure distances, angles, etc. using a Riemannian metric. In Riemannian optimization \citep{absil2008optimization, boumal2023introduction}, variables are constrained to lie on a smooth manifold $\mathcal{M}$ (such as a sphere), rather than a flat Euclidean space $\mathbb{R}^n$. The core concepts for Riemannian optimization are thus:

\begin{itemize}

    \item \textbf{The Manifold ($\mathcal{M}$):} The search space.
    \item \textbf{Tangent Space ($T_x\mathcal{M}$):} At any point $x \in \mathcal{M}$, the tangent space is a linear approximation of the manifold's local geometry. Gradients are computed in this space.
    \item \textbf{Riemannian Gradient:} The projection of the standard Euclidean gradient onto the tangent space, denoted as $\operatorname{grad} f(x) = \text{P}_{T_x\mathcal{M}}(\nabla f(x))$.
    \item \textbf{Retraction:} A mapping $R_x: T_x\mathcal{M} \to \mathcal{M}$ that moves an element from the tangent space back onto the manifold.

\end{itemize}

LORO searches on the \textbf{low-rank matrix manifold}, defined as the set of matrices with a fixed rank $r$. Each matrix $W$ has a \textit{unique representation}, eliminating redundancy during the search process. However, the retraction operation requires an SVD operation, which is computationally expensive. The cost, in the number of FLOPs, for a matrix $W \in \mathbb{R}^{m \times n}$ is
\begin{equation} \label{eq:svd_flops}
mn^2 + n^3.
\end{equation}
To manage this cost, LORO executes the SVD retraction sparingly, every $K = 500$ steps. In contrast, the QR decomposition in TSVD optimizes over the \textbf{Stiefel manifold}, i.e.\ the set of all matrices with orthonormal columns, denoted as
\begin{equation} \label{eq:stiefel}
St(n,r) = \{ X \in \mathbb{R}^{n \times r} \mid X^\top X = I_r \}.
\end{equation}
This approach may result in \textit{non-unique representations}, since different factors ($U, \Sigma, V$) can represent the same product $W$, which could be a limitation. However, the QR decomposition is significantly more efficient than the SVD operation (compare Eq.~\ref{eq:qr_flops} and~\ref{eq:svd_flops}).
Thus, the assumption of TSVD is that the computational efficiency of the projection onto the Stiefel manifold outweighs the theoretical benefits of unique representation on the fixed-rank manifold. Indeed, the results in this paper suggest that, given the massive scale of current and future LLMs, the significant reduction in retraction overhead justifies the inherent redundancy in factorized representation, making TSVD a more practical approach for large-scale optimization than other current methods.

\begin{figure}[!t]
    \centering
    \includegraphics[width=\linewidth]{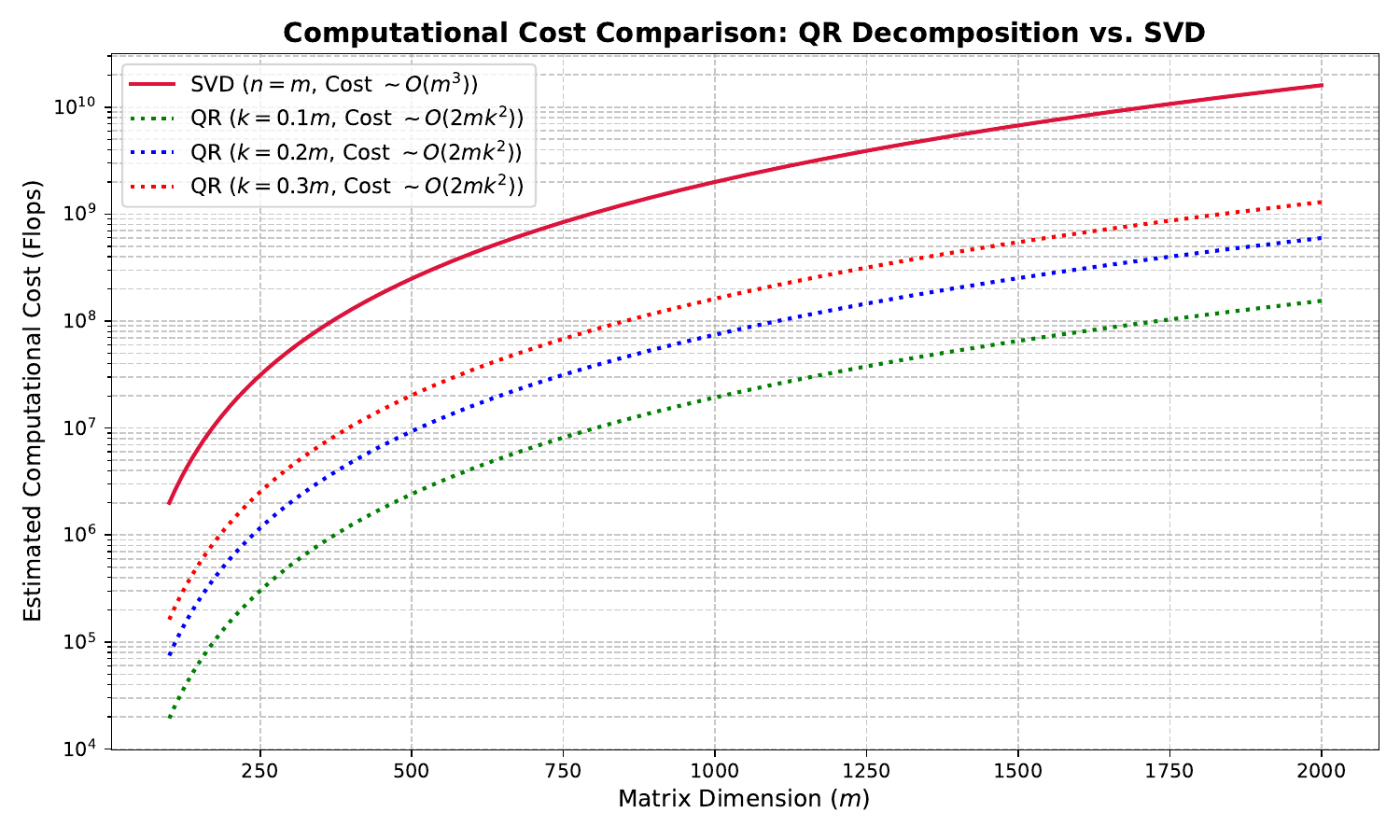}
    \caption{Comparison of the computational cost for QR decomposition of $W_q \in \mathbb{R}^{m \times k}$ matrix with various $k$ values and SVD for $W_s \in \mathbb{R}^{m \times n}$ matrix with $m=n$. The plot illustrates that QR decomposition requires significantly fewer FLOPs.}
    \label{fig:qr_vs_svd}
\end{figure}

\section{Full-Parameter vs Rank-\textit{k} Memory and Compute Comparison}
\label{app:full_vs_rank_k_flops}

This section compares the memory and compute requirements of full-parameter weight matrices and rank-$k$ representation as used in TSVD. It shows that the TSVD representation reduces both the number of FLOPs and the memory requirements. For reference, the formulas for memory and compute comparison of full rank vs rank-\textit{k} are summarized in Tab.~\ref{table:memory_compute_comparison}.

\paragraph{Memory comparison.}
For a full-parameter weight matrix $W \in \mathbb{R}^{m \times n}$, $mn$ elements are stored, where each element can take a few bytes depending on their format (4 bytes for FP32, or 2 bytes for FP16, etc.).

For $W \approx U_k \Sigma_k V_k^\top$ with $U_k \in \mathbb{R}^{m \times k}$, $\Sigma_k \in \mathbb{R}^{k \times k}$, and $V_k \in \mathbb{R}^{n \times k}$, $mk$, $k$, and $nk$ elements are stored for unconstrained parameters, i.e., $k(m+n+1) \approx k(m+n)$, and the same amount for constrained parameters (i.e. QR decomposition on $U_k$ and $V_k$, and softplus on $\Sigma$). Hence, the total memory requirement is $2k(m+n)$.

As a result, rank-\textit{k} needs less memory than full rank when 
\begin{equation}
2k(m+n) < mn,
\end{equation}
that is,
\begin{equation}
k < \frac{mn}{2(m+n)}.
\end{equation}

\paragraph{Example.}

To compare the memory requirements of full-parameter and rank-\textit{k} weight matrices, suppose $k=128$, and $n=m=1024$. Full-parameter weight matrix requires memory for $m * n = 1,048,576$ floating point values. Rank-\textit{k} matrix requires memory for $2k(m+n) = 2* 128 *(1024 + 1024) = 524,288$ floating point values, which is a $50\%$ reduction in memory requirements.

\paragraph{Compute Comparison.}
For an input $x \in \mathbb{R}^n$ and a weight matrix $W \in \mathbb{R}^{m \times n}$, computing $y = Wx$ in the forward pass requires $mn$ multiply-add operations thus $2mn$ FLOPs. During backpropagation, computing both the weight gradient $\partial W$ and the input gradient $\partial x$ requires approximately $4mn$ FLOPs. Therefore, the total cost of one forward pass and one backward pass through the layer is approximately $6mn$ FLOPs.\\

For $W \approx U_k \Sigma_k V_k^\top$ with $U_k \in \mathbb{R}^{m \times k}$, $\Sigma_k \in \mathbb{R}^{k \times k}$, and $V_k \in \mathbb{R}^{n \times k}$, computing $y = Wx$ in the forward pass costs
\begin{equation}
2kn \; (\text{for} \; z = V_k^\top x) \; + \;
k \; (\text{for} \; z' = \Sigma_k z) \; + \;
2mk \; (\text{for} \; y = U_k z').
\end{equation}
We ignore the FLOPs for $\Sigma_k$, both in the forward and backward pass, as it is negligible due to its efficient computation as element-wise vector product. The total cost for the forward pass is therefore $\approx 2k(n + m)$ FLOPs. Computing the QR decomposition in the forward pass costs $\approx 2nk^2$ FLOPs for an $n$ by $k$ matrix, but it only happens once per update step as the calculated value is cached and re-used. During backpropagation, computing both the weight gradient $\partial W$ and the input gradient $\partial x$ requires approximately $4nk$ FLOPs for $V_k$ and $4mk$ for $U_k$, totaling $4k(n+m)$. Therefore, the total cost of one forward pass and one backward pass through the layer is approximately $6k(m+n) + 2k^2(m+n)$ FLOPs.\\

In LLM pretraining, gradient accumulation is a standard practice. This further reduces the influence of QR decomposition on the overall FLOPs. Given a batch size of $B$, sequence length of $S$, and number of accumulation steps of $t$, the number of FLOPs for a weight matrix $W \in \mathbb{R}^{m \times n}$ is
\begin{equation}
6tBSmn,
\end{equation}
and the number of FLOPs for $W \approx U_k \Sigma_k V_k^\top$ with $U_k \in \mathbb{R}^{m \times k}$, $\Sigma_k \in \mathbb{R}^{k \times k}$, and $V_k \in \mathbb{R}^{n \times k}$ is 
\begin{equation}
tBS(6k(m+n)) + 2k^2(m+n),
\end{equation}
where the FLOPs for $\Sigma_k$ are ignored (because it is a negligible cost). The second term is the QR decomposition FLOPs for $V_k$ and $U_k$. 

\begin{table}[t]
\centering
\setlength{\tabcolsep}{8pt}
\caption{Resource comparison for weight matrix $W \in \mathbb{R}^{m \times n}$ and $W \approx U_k \Sigma_k V_k^\top$ with $U_k \in \mathbb{R}^{m \times k}$, $\Sigma_k \in \mathbb{R}^{k \times k}$, and $V_k \in \mathbb{R}^{n \times k}$, where $B$ is the batch size, $S$ is the sequence length, and $t$ is the number of gradient accumulation steps.}
\fontsize{8pt}{9pt}\selectfont
\label{table:memory_compute_comparison}
\begin{tabular}{ccc}
\toprule
\multicolumn{1}{c}{}  &\multicolumn{1}{c}{\bf Full Rank} &\multicolumn{1}{c}{\bf Rank-\textit{k}}
\\ \midrule
\textbf{Memory} & $mn$     & $2k(m+n)$ \\ 
\textbf{Flops}  & $6tBSmn$ & $tBS(6k(m+n)) + 2k^2(m+n)$  \\
\bottomrule
\end{tabular}
\end{table}

\paragraph{Example.}
To compare the number of FLOPs for of full rank and rank-\textit{k}, suppose $B=512$, $S=256$, $k=128$, and $n=m=1024$. We calculate the number of FLOPs for three numbers of accumulation steps: $t=1, 8$, and  $16$. 

The \textit{full-parameter} weight matrix incurs
\begin{equation}
824.63 \times 10^9, \quad 6.59 \times 10^{12}, \quad 13.19 \times 10^{12}
\end{equation}
FLOPs for $t=1$, $t=8$ and $t=16$ respectively.

\textit{Rank-k} matrices incus the following FLOPs at accumulation step $t=1$, where the second term is the QR decomposition FLOPs:
\begin{equation}
206.15 \times 10^9 + 67.1  \times 10^6.
\end{equation}

For accumulation step $t=8$, it is
\begin{equation}
1.64  \times 10^{12} + 67.1 \times 10^{6}.
\end{equation}

And, for accumulation step $t=16$, it is
\begin{equation}
32.9 \times 10^{12} + 67.1 \times 10^{6}.
\end{equation}

\paragraph{Conclusion.}
Thus, (1) Even for accumulation step $t=1$, the QR decomposition FLOPs are a small proportion of the overall FLOPs, and (2) As the number of accumulation steps increases to 8 or 16, the QR decomposition FLOPs become an even smaller proportion of the overall FLOPs. Rank-$k$ weight matrices + QR decomposition as used in the TSVD framework thus offer a strong profile in terms of memory and compute for sufficiently small $k$.

\section{Compression Rank Analysis for Open-Weight Model Families} \label{sec:rank_analysis}

Sec.~\ref{sec:effective_rank_analysis} in the main paper analyzed the compression ranks (Eq.~\ref{eq:effective_rank}) of individual linear projections (query, key, value, out for the attention layer; gate, up, down for the MLP layer) within individual transformer layers for the full-parameters models trained for benchmarking in Sec.~\ref{sec:experiments}.

To assess the generality of the findings and highlight potential deviations, the following analysis extends to large, publicly available open-weight model families:
Llama-2 \citep{touvron2023llama}, Llama-3.1 \citep{dubey2024llama}, and Olmo 3 \citep{olmo2026olmo3}. The architectural parameters of these models are listed in Tab.~\ref{tab:parameters_main_families}. The query heads and key/value heads were treated separately, because many modern architectures employ grouped query attention (GQA) \citep{ainslie-etal-2023-gqa}.

\paragraph{Results.}
Compression ranks per transformer layer and projection type are shown in Fig.~\ref{fig:effective_rank_llama2}~-~\ref{fig:effective_rank_olmo3} for completeness.
Except for the very first layer, they are relatively consistent per projection type.
Following the analysis in Sec.~\ref{sec:effective_rank_analysis} normalized compression ranks (averaged over all transformer layers) are reported in Fig.~\ref{fig:rank_comparison}, which are discussed in detail in the following. 
For Llama-2 7B and 13B, the same constant scaling behavior as in the analysis on the benchmark models in Sec.~\ref{sec:effective_rank_analysis} is observed.  
This result was expected given that all major architectural parameters are scaled proportionally.
In contrast, the 70B variant adopts GQA, reducing the number of key and value heads by a factor of four. 
This change is directly reflected in the compression ranks of key and value projections, which decrease relative to the query and output projections.
Additionally, the MLP intermediate dimension is scaled disproportionately, resulting in higher compression ranks for the gate, up, and down projections.
For Llama-3.1, the number of key and value heads remain constant while most other parameters double from 8B to 70B.
Correspondingly, the normalized compression ranks of key and value heads are approximately halved.
The transition from OLMo 3 7B to 32B exhibits a similar pattern to that observed between Llama-2 13B and 70B.

\paragraph{Conclusion.}
Overall, we find that the findings of Sec.~\ref{sec:effective_rank_analysis} generalize to common open-weight families.
However, upon changing critical architectural parameters disproportionally when scaling to larger models, e.g. through introducing GQA, rank selection needs to be adapted accordingly.

\begin{figure}[!t]
\centering
\includegraphics[width=0.7\linewidth]{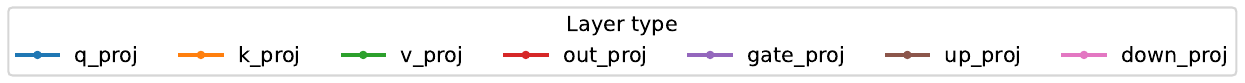}\\[3pt]
\begin{subfigure}[b]{0.49\linewidth}
    \centering
    \includegraphics[width=\linewidth, trim=0 0 0 0.8cm, clip]{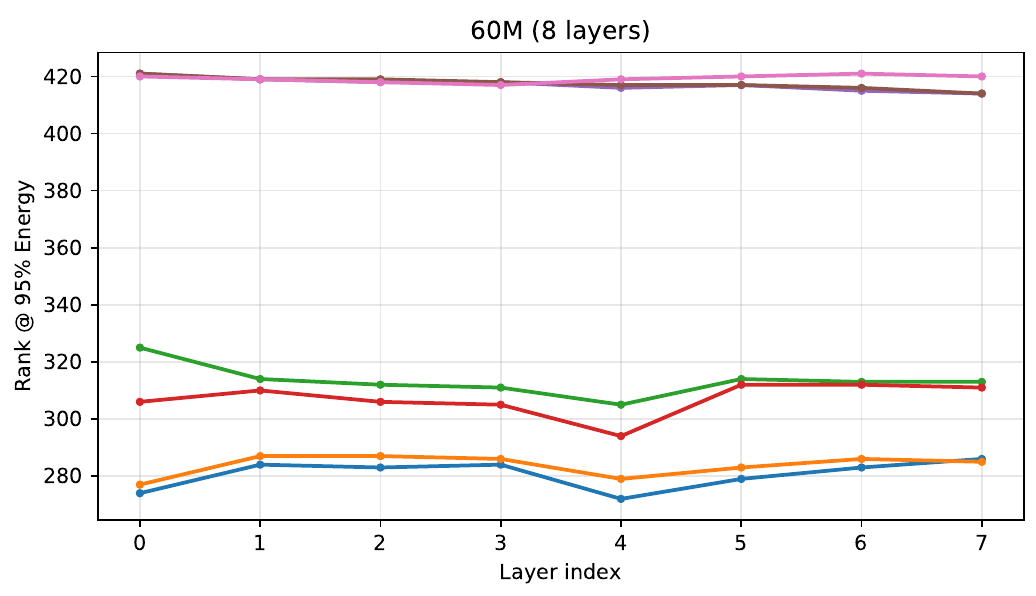}
    \caption{60M (8 layers)}
    \label{fig:effective_rank_main_a}
\end{subfigure}
\hfill
\begin{subfigure}[b]{0.49\linewidth}
    \centering
    \includegraphics[width=\linewidth, trim=0 0 0 0.8cm, clip]{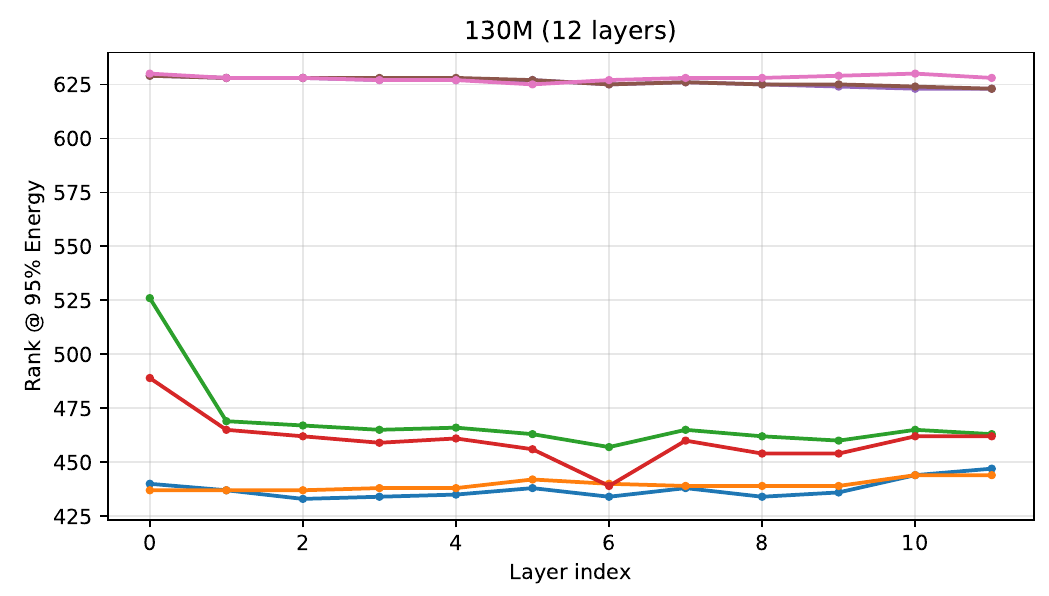}
    \caption{130M (12 layers)}
    \label{fig:effective_rank_main_b}
\end{subfigure}

\vspace{4pt}

\begin{subfigure}[b]{0.49\linewidth}
    \centering
    \includegraphics[width=\linewidth, trim=0 0 0 0.6cm, clip]{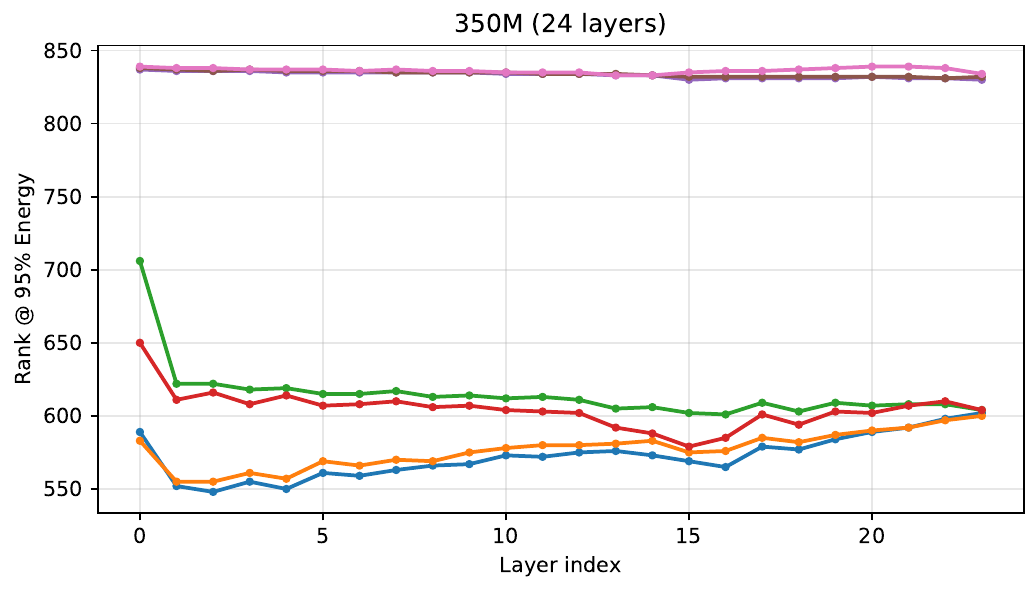}
    \caption{350M (24 layers)}
    \label{fig:effective_rank_main_c}
\end{subfigure}
\hfill
\begin{subfigure}[b]{0.49\linewidth}
    \centering
    \includegraphics[width=\linewidth, trim=0 0 0 0.6cm, clip]{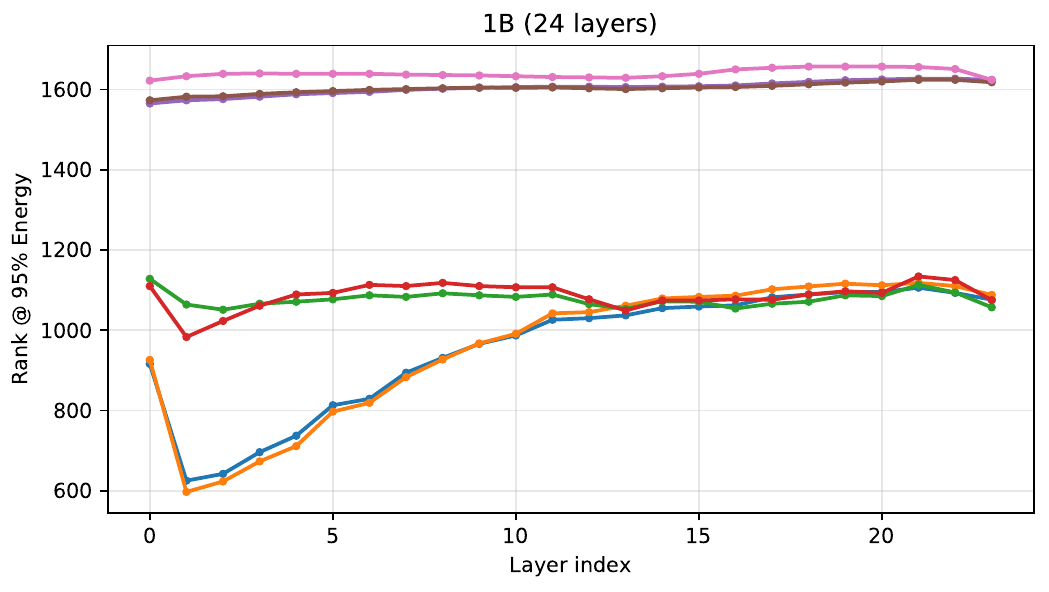}
    \caption{1B (24 layers)}
    \label{fig:effective_rank_main_d}
\end{subfigure}

\caption{Rank required to capture 95\% of the spectral energy for each weight type (query, key, value, out, gate, up, and down projections), plotted across transformer layers, for four models from the benchmark of Sec.~\ref{sec:pretraining}: (a) 60M (8 layers), (b) 130M (12 layers), (c) 350M (24 layers), and (d) 1B (24 layers) parameters. Within every model, the required rank is nearly constant across layer index once the first layer is excluded, and the relative ordering of weight types (e.g., gate/up/down projections consistently needing higher rank than query/key/value/out) is preserved across scales. This uniformity is what justifies assigning a single compression rank per weight type, rather than tuning a separate rank for every individual layer, without sacrificing the accuracy of the energy-based rank selection.}
\label{fig:effective_rank_main}
\end{figure}

\begin{table}[!t]
\centering
\setlength{\tabcolsep}{4.4pt}
\caption{Major architectural parameters of open-weights model families. Llama-2, Llama-3.1 and OLMo 3 models of different sizes are compared with respect to their embedding dimension $d_1$, their hidden MLP dimension $d_2$, the number of layer $n_{\text{layers}}$, number of query heads $n_{\text{heads,q}}$ and number of key and value heads $n_{\text{heads,kv}}$. Grey numbers denote multiples compared to the smallest model within the same model family. Except for Llama-2 7B to 13B, parameters change non-proportionally. Most notable, $n_{\text{heads,kv}}$ often decrease due to GQA. Fig.~\ref{fig:rank_comparison} illustrates how normalized compression ranks are influenced by these parameters, demonstrating a consistent scaling trend that enables their prediction in larger models from smaller ones, with the exception of key and value projection ranks, which decrease due to GQA.}
\fontsize{8pt}{9pt}\selectfont
\label{tab:parameters_main_families}
\begin{tabular}{c c c cc cc c c cc c c cc}
\toprule
Parameter & 
& \multicolumn{5}{c}{\hspace{0.2cm}Llama-2} &
& \multicolumn{3}{c}{\hspace{0.2cm}Llama-3.1} & 
& \multicolumn{3}{c}{\hspace{0.2cm}OLMo 3} \\

\cmidrule(r){1-2} \cmidrule(lr){3-8} \cmidrule(lr){9-12} \cmidrule(l){13-15}

num\_parameters & 
& 7B
& 13B & \textcolor{gray}{$\times$ 7B}
& 70B & \textcolor{gray}{$\times$ 7B} &
& 8B
& 70B & \textcolor{gray}{$\times$ 8B} &
& 7B
& 32B & \textcolor{gray}{$\times$ 7B} \\

\cmidrule{1-15}

$d_1$ &
& 4096 
& 5120 & \textcolor{gray}{1.25}
& 8192 & \textcolor{gray}{2} &
& 4096 
& 8192 & \textcolor{gray}{2} &
& 4096 
& 5120 & \textcolor{gray}{1.25} \\

$d_2$ &
& 11008 
& 13824 & \textcolor{gray}{1.26}
& 28672 & \textcolor{gray}{2.6} &
& 14336 
& 28672 & \textcolor{gray}{2} &
& 11008 
& 27648 & \textcolor{gray}{2.5} \\

$n_{\text{layers}}$ &
& 32 
& 40 & \textcolor{gray}{1.25}
& 80 & \textcolor{gray}{2.5} &
& 32 
& 80 & \textcolor{gray}{2.5} &
& 32 
& 64 & \textcolor{gray}{2} \\

$n_{\text{heads,q}}$ &
& 32 
& 40 & \textcolor{gray}{1.25}
& 64 & \textcolor{gray}{2} &
& 32 
& 64 & \textcolor{gray}{2} &
& 32 
& 40 & \textcolor{gray}{1.25} \\

$n_{\text{heads,kv}}$ &
& 32 
& 40 & \textcolor{gray}{1.25}
& 8 & \textcolor{gray}{0.25} &
& 8 
& 8 & \textcolor{gray}{1} &
& 32 
& 8 & \textcolor{gray}{0.25} \\

\bottomrule
\end{tabular}
\end{table}

\begin{figure}
\begin{subfigure}[b]{\linewidth}
    \centering
    \includegraphics[width=0.7\linewidth, trim=0 -0.2cm 0 0, clip]{figures/energy/main_exp_95_raw_ranks_legend.pdf}
\end{subfigure}
\begin{subfigure}[b]{0.33\linewidth}
    \centering
    \includegraphics[width=\linewidth, trim=0cm 0.2cm 35.5cm 2.5cm, clip]{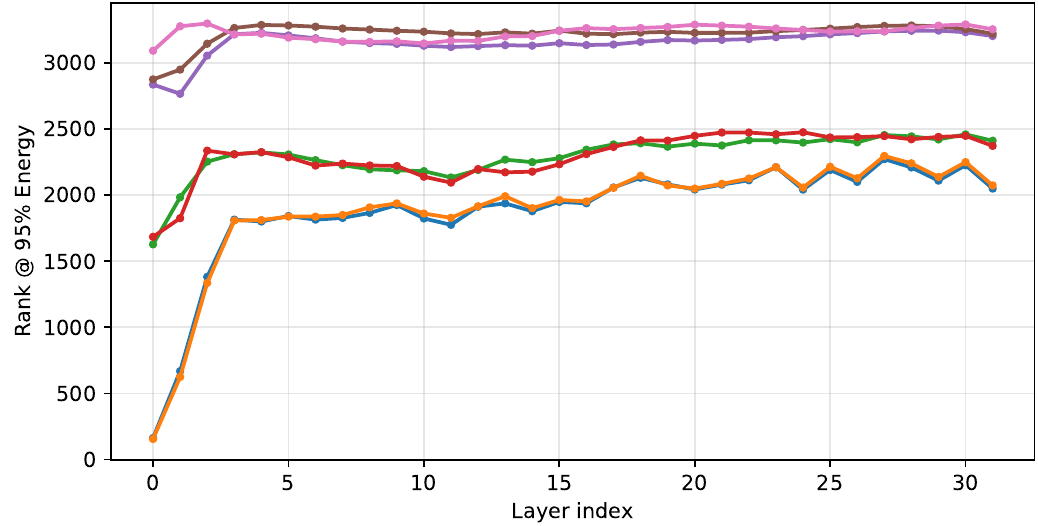}
    \caption{7B (32 layers)}
\end{subfigure}
\begin{subfigure}[b]{0.33\linewidth}
    \centering
    \includegraphics[width=\linewidth, trim=17.8cm 0.2cm 17.6cm 2.5cm, clip]{figures/energy/llama2_95_raw_ranks.pdf}
    \caption{13B (40 layers)}
\end{subfigure}
\begin{subfigure}[b]{0.33\linewidth}
    \centering
    \includegraphics[width=\linewidth, trim=35.5cm 0.2cm 0cm 2.5cm, clip]{figures/energy/llama2_95_raw_ranks.pdf}
    \caption{70B (80 layers)}
\end{subfigure}
\caption{Compression ranks for different weight types (query, key, value, out, gate, up and down projections) across transformer layers for different model sizes of the Llama-2 model family. Except for the first layers, compression ranks are relatively uniform.}
    \label{fig:effective_rank_llama2}
\end{figure}

\begin{figure}
\begin{subfigure}[b]{\linewidth}
    \centering
    \includegraphics[width=0.7\linewidth, trim=0 -0.2cm 0 0, clip]{figures/energy/main_exp_95_raw_ranks_legend.pdf}
\end{subfigure}
\begin{subfigure}[b]{0.49\linewidth}
    \centering
    \includegraphics[width=\linewidth, trim=0cm 0.2cm 17.7cm 2.5cm, clip]{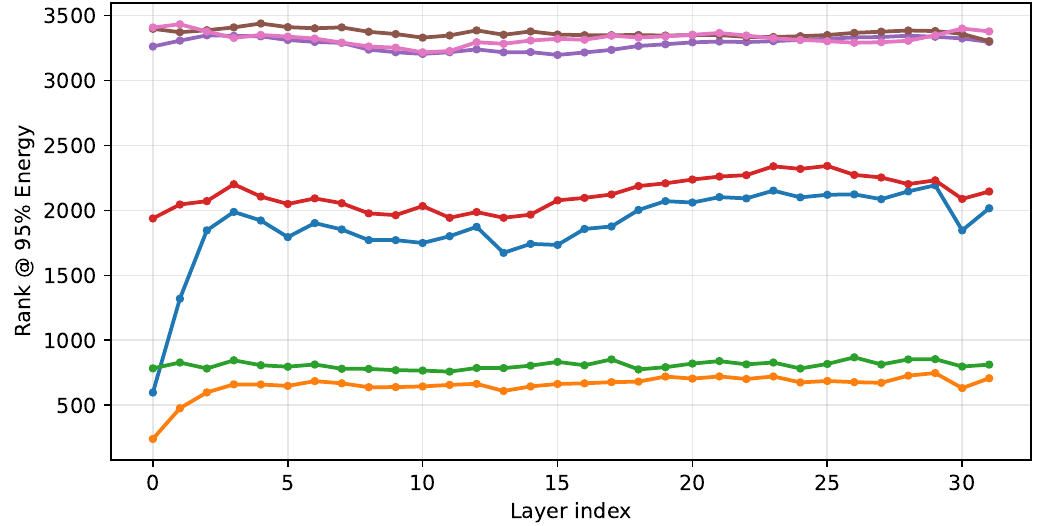}
    \caption{8B (32 layers)}
\end{subfigure}
\begin{subfigure}[b]{0.49\linewidth}
    \centering
    \includegraphics[width=\linewidth, trim=17.7cm 0.2cm 0cm 2.5cm, clip]{figures/energy/llama3_95_raw_ranks.pdf}
    \caption{70B (80 layers)}
\end{subfigure}
    \caption{Compression ranks for different weight types (query, key, value, out, gate, up and down projections) across transformer layers for different model sizes of the Llama-3.1 model family. Except for the first layers, compression ranks are relatively uniform.}
    \label{fig:effective_rank_llama3}
\end{figure}

\begin{figure}
\begin{subfigure}[b]{\linewidth}
    \centering
    \includegraphics[width=0.7\linewidth, trim=0 -0.2cm 0 0, clip]{figures/energy/main_exp_95_raw_ranks_legend.pdf}
\end{subfigure}
\begin{subfigure}[b]{0.49\linewidth}
    \centering
    \includegraphics[width=\linewidth, trim=0cm 0.2cm 17.7cm 2.5cm, clip]{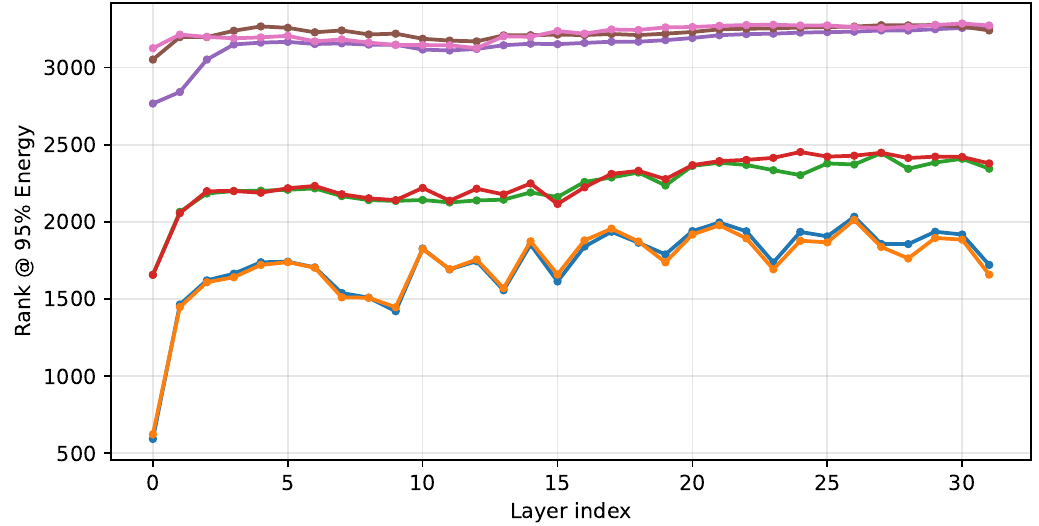}
    \caption{7B (32 layers)}
\end{subfigure}
\begin{subfigure}[b]{0.49\linewidth}
    \centering
    \includegraphics[width=\linewidth, trim=17.7cm 0.2cm 0cm 2.5cm, clip]{figures/energy/olmo3_95_raw_ranks.pdf}
    \caption{32B (64 layers)}
\end{subfigure}
\caption{Compression ranks for different weight types (query, key, value, out, gate, up and down projections) across transformer layers for different model sizes of the OLMo-3 model family. Except for the first layers, compression ranks are relatively uniform.}
    \label{fig:effective_rank_olmo3}
\end{figure}

\begin{figure}
\begin{subfigure}[b]{\linewidth}
    \centering
    \includegraphics[width=0.7\linewidth, trim=0 -0.2cm 0 0, clip]{figures/energy/main_exp_95_raw_ranks_legend.pdf}
\end{subfigure}
\begin{subfigure}[b]{0.37\linewidth}
    \centering
    \includegraphics[width=\linewidth, trim=0cm 0.2cm 15.4cm 0.55cm, clip]{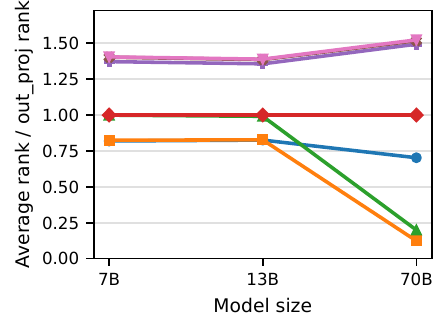}
    \caption{Llama 2}
\end{subfigure}
\begin{subfigure}[b]{0.31\linewidth}
    \centering
    \includegraphics[width=\linewidth, trim=7.6cm 0.2cm 9.0cm 0.55cm, clip]{figures/energy/95_energy_rank_comparison.pdf}
    \caption{Llama 3.1}
\end{subfigure}
\begin{subfigure}[b]{0.31\linewidth}
    \centering
    \includegraphics[width=\linewidth, trim=13.9cm 0.2cm 2.7cm 0.55cm, clip]{figures/energy/95_energy_rank_comparison.pdf}
    \caption{OLMo 3}
\end{subfigure}
\caption{Normalized compression ranks (averaged over layers) of different model sizes within different model families. Ranks remain similar except that key and value projection ranks often decrease due to GQA.}
    \label{fig:rank_comparison}
\end{figure}

\end{document}